%% file: MAIN.tex
\documentclass[twoside,11pt]{article}

\usepackage{blindtext}
\usepackage{jmlr2e}
\usepackage{float}
\usepackage{placeins}

\input{macros_alt}

\usepackage{graphicx}
\graphicspath{{figs/}}  

\DeclareGraphicsExtensions{.pdf,.png,.jpg}  
\usepackage{caption}

\newcommand{\curly}[1]{\left\{#1\right\}}

\def\x{\bm{x}}
\def\y{\bm{y}}

\definecolor{shadecolor}{gray}{0.9}
\definecolor{linkblue}{HTML}{5DADE2}
\definecolor{linkpink}{HTML}{D16BA5}

\def\x{\bm{x}}
\def\w{\bm{w}}
\def\z{\bm{z}}
\def\y{\bm{y}}

\def\mR{\bm{R}}

\newtoggle{longversion}
\settoggle{longversion}{true}
\hypersetup{colorlinks=true,linkcolor=linkpink,citecolor=linkblue,urlcolor=linkblue}

\usepackage{times}

\ShortHeadings{}{Decision Trees vs Neural Networks}
\firstpageno{1}

\makeatletter
\def\@maketitle{\vbox{\hsize\textwidth
 \linewidth\hsize \vskip \beforetitskip
 {\begin{center} \Large\bf \@title \par \end{center}} \vskip \aftertitskip
 {\begin{center}%
  \def\and{\unskip\enspace{\rm and}\enspace}%
  \def\addr{\small\it}%
  \def\email{\small\sc}
  \def\name{\normalsize\bf}%
  \def\AND{\@endauthor\rm\hss \vskip \interauthorskip \@startauthor}
  \@startauthor \@author \@endauthor
  \end{center}}%
  \vskip \aftermaketitskip
  \if@preprint
  \else
  \vskip \aftermaketitskip
  \fi
}}
\makeatother

\begin{document}

\title{A Gap Between Decision Trees and Neural Networks}

\author{\name Akash Kumar \\\email (akk002@ucsd.edu) \\
       \addr Department of Computer Science \& Engineering\\
       University of California-San Diego\\
       La Jolla, CA 92093-0404, USA
       }


\maketitle
\vspace{-5mm}
\begin{abstract}
We study when geometric simplicity of decision boundaries, used here as a notion of interpretability, can conflict with accurate approximation of axis-aligned decision trees by shallow neural networks. Decision trees induce rule-based, axis-aligned decision regions (finite unions of boxes), whereas shallow ReLU networks are typically trained as score models whose predictions are obtained by thresholding. We analyze the infinite-width, bounded-norm, single-hidden-layer ReLU class through the Radon total variation ($\mathcal{R}\mathrm{TV}$) seminorm, which controls the geometric complexity of level sets.

We first show that the hard tree indicator $1_A$ has infinite $\mathcal{R}\mathrm{TV}$. Moreover, two natural split-wise continuous surrogates--piecewise-linear ramp smoothing and sigmoidal (logistic) smoothing--also have infinite $\mathcal{R}\mathrm{TV}$ in dimensions $d>1$, while Gaussian convolution yields finite $\mathcal{R}\mathrm{TV}$ but with an explicit exponential dependence on $d$.

We then separate two goals that are often conflated: classification after thresholding (recovering the decision set) versus score learning (learning a calibrated score close to $1_A$). For classification, we construct a smooth barrier score $S_A$ with finite $\mathcal{R}\mathrm{TV}$ whose fixed threshold $\tau=1$ exactly recovers the box. Under a mild tube-mass condition near $\partial A$, we prove an $L_1(P)$ calibration bound that decays polynomially in a sharpness parameter, along with an explicit $\mathcal{R}\mathrm{TV}$ upper bound in terms of face measures. Experiments on synthetic unions of rectangles illustrate the resulting accuracy--complexity tradeoff and how threshold selection shifts where training lands along it.
\end{abstract}

\input{intro}
\input{related_work}

\input{problem_setup}
\input{hard-threshold}
\input{smoothing}
\input{threshold_class}

\input{conclusion}

\newpage
\bibliography{ref}
\newpage
\appendix
\input{toc}

\input{app_DT}

\input{app_Sigmoid}

\input{app_Gaussian}

\input{app_class}

\end{document}

%% file: macros_alt.tex
\usepackage{savesym,url,hyperref,color}
\usepackage[titletoc]{appendix}
\usepackage{multicol}
\usepackage{blkarray}
\usepackage{wrapfig}

\usepackage{etoolbox}
\usepackage{hyperref}       
\usepackage{url}            
\usepackage{booktabs}       
\usepackage{tikz}
\usepackage{pgfplots}
\usepackage{verbatim}
\usepackage{amsmath,bm}
\usepackage{multirow}
\usepackage{mathtools}
\usepackage{comment}

\newtoggle{showcomments}
\settoggle{showcomments}{true}

\iftoggle{showcomments}{
\newcommand{\akash}[1]{{\textcolor{purple}{[{\bf A:} #1]}}}
\newcommand{\parthe}[1]{{\textcolor{red}{[{\bf A:} #1]}}}
\newcommand{\annotate}[1]{{\textcolor{blue}{[{\bf Note:} #1]}}}
}
{
\newcommand{\akash}[1]{}
\newcommand{\parthe}[1]{}
\newcommand{\annotate}[1]{}
}

\newtheorem{assumption}{Assumption}

\makeatletter

\newcommand{\Rmnum}[1]{\expandafter\@slowromancap\romannumeral #1@}
\makeatother
\usepackage{caption}

\newcommand{\figref}[1]{Fig.~\ref{#1}}
\newcommand{\eqnref}[1]{\text{Eq.}~(\ref{#1})}
\newcommand{\secref}[1]{\textnormal{Section}~\ref{#1}}
\newcommand{\appref}[1]{Appendix \ref{#1}}

\newcommand{\thmref}[1]{Theorem~\ref{#1}}

\newcommand{\lemref}[1]{Lemma~\ref{#1}}

\newcommand{\rtv}[1]{\mathcal{R}\mathrm{TV}^2({#1})}
\newcommand{\rbv}[1]{\mathcal{R}\mathrm{BV}^2({#1})}

\newcommand{\paren} [1] {\ensuremath{ \left( {#1} \right) }}

\newcommand{\bracket}[1]{\left[#1\right]}

\newcommand{\curlybracket}[1]{\ensuremath{\left\{#1\right\}}}

\newcommand{\reals}{\ensuremath{\mathbb{R}}}

\newcommand{\cS}{{\mathcal{S}}}
\newcommand{\cR}{{\mathcal{R}}}

\newcommand{\cB}{{\mathcal{B}}}

\newcommand{\cX}{{\mathcal{X}}}
\newcommand{\cH}{{\mathcal{H}}}

\newcommand{\be}{{\mathbf{e}}}

\newcommand{\bw}{{\mathbf{w}}}

\renewcommand{\sf}[1]{\textsf{#1}}
\def\BState{\State\hskip-\ALG@thistlm}

\newcommand{\tv}{\mathcal{R}\mathrm{TV}^2}
\newcommand{\rv}{\mathcal{RBV}}

\definecolor{cyanblue}{RGB}{170, 23, 205}

\definecolor{redblue}{RGB}{10, 23, 205}

%% file: intro.tex
\section{Introduction}\label{sec:intro}

In safety-critical and socially sensitive applications, it is often desirable to deploy predictors whose behavior can be explained and audited. A common interpretable baseline is the axis-aligned decision tree: it classifies $x\in\mathbb{R}^d$ by a sequence of one-dimensional threshold tests. Equivalently, it induces a piecewise-constant classifier of the form

\[
\hspace{-6mm}f_{\mathrm{DT}}(x) \;=\; \mathbf{1}\{x\in A\}, \hspace{5mm}
\]
where $A\subseteq\mathbb{R}^d$ is a finite union of axis-aligned boxes.

\begin{wrapfigure}{r}{0.38\textwidth}
  \vspace{-6.4em}
  \centering
  \includegraphics[width=\linewidth]{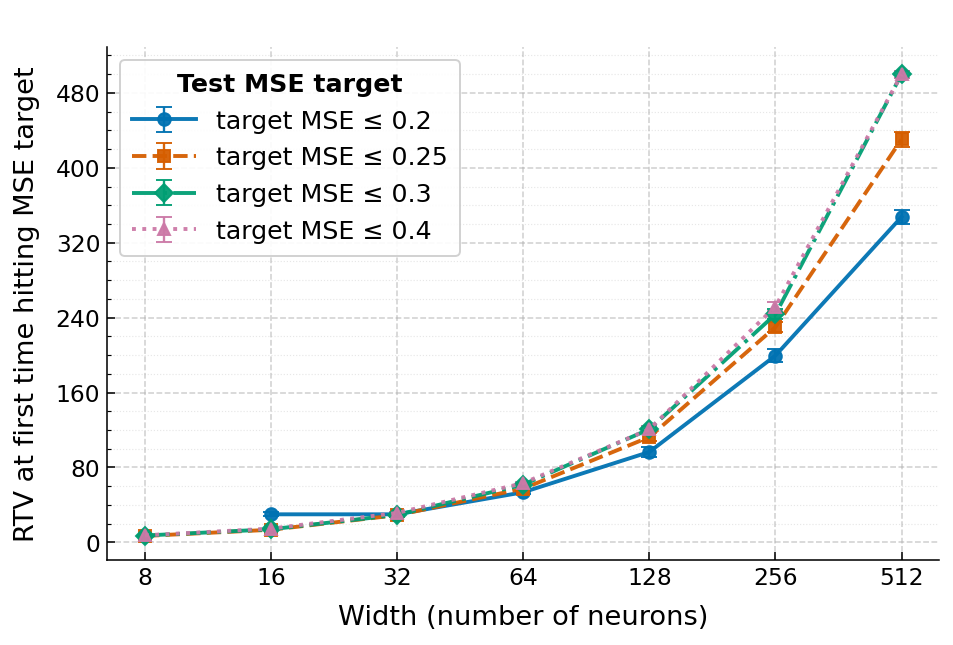}
  \caption{Width–$\mathcal{R}\mathrm{TV}$ frontier (depth-1 ReLU, box task).}
  \label{fig:intro-rtv-frontier}
    \vspace{-2.1em}
\end{wrapfigure}

In contrast, a single-hidden-layer network forms a score
\[
x \;\mapsto\; \sum_{k=1}^K v_k\,\sigma(w_k^\top x-b_k),
\]
and typically produces a classifier only after applying a nonlinearity and/or threshold. We analyze the infinite-width, bounded-norm counterpart of this model through the \emph{Radon total variation} ($\mathcal{R}\mathrm{TV}$) seminorm $\|\cdot\|_{\mathcal{R}}$ \citep{savarese19a,Ongie2020A,Parhi2019TheRO}. Intuitively, $\|\cdot\|_{\mathcal{R}}$ controls the amount of geometric complexity in level sets of a function and coincides with the minimum-mass ridgelet (infinite-width ReLU) representation norm. 
This makes $\mathcal{R}\mathrm{TV}$ a natural lens for asking when shallow networks can represent tree-like decision structure without becoming geometrically complex.

\paragraph{Two learning goals that should be separated.}
There are (at least) two distinct objectives when ``learning a tree'' with a shallow net:
\begin{enumerate}
\item \textbf{Classification via thresholding.} Learn a score $s:\mathbb{R}^d\to[0,1]$ such that $\{x:s(x)\ge \tau\}=A$ for some threshold $\tau$ (possibly fixed or tuned on validation data).
\item \textbf{Score learning / regression.} Learn a calibrated score close to $1_A$ (e.g., in $L_1(P)$), so that the score's level sets and gradients align with the symbolic features of the tree.
\end{enumerate}
These goals are often conflated in empirical evaluations: many very different score functions can induce the same thresholded classifier, so near-perfect classification accuracy can mask large differences in calibration and geometry.

\paragraph{This paper.}
We formalize the distinction above and show that it matters sharply under $\mathcal{R}\mathrm{TV}$ control. Structurally, we prove that the hard tree indicator $1_A$ has infinite $\mathcal{R}\mathrm{TV}$. We then show that several natural continuous surrogates retain this pathology in $d>1$: ramp smoothing is continuous piecewise-linear but typically has infinite $\mathcal{R}\mathrm{TV}$ when multiple split normals are present, and sigmoidal smoothing has infinite $\mathcal{R}\mathrm{TV}$ as soon as the depth is at least two. In contrast, Gaussian smoothing yields a finite $\mathcal{R}\mathrm{TV}$ bound, but with an explicit exponential dependence on the ambient dimension.

On the learning side, we show that \emph{classification itself is easy} if one only requires correct thresholded decisions: we construct a smooth barrier score $S_B$ for an axis-aligned box $B$ such that a fixed cutoff $\tau=1$ exactly recovers $B$. Under a mild tube-mass condition near $\partial B$, $S_B$ is also $L_1(P)$-close to $\mathbf{1}_B$ with a polynomial rate in a sharpness parameter, and it has an explicit finite $\mathcal{R}\mathrm{TV}$ bound. Thus, shallow nets can match tree decisions via thresholding without approximating the discontinuous tree function in any strong sense. In contrast, insisting on a score that is both calibrated and geometrically simple reveals a quantitative accuracy--complexity frontier.

\paragraph{Contributions (informal).}
Let $A$ be a finite union of axis-aligned boxes in $\mathbb{R}^d$.
\begin{enumerate}
\item \textit{Hard trees lie outside bounded-$\mathcal{R}\mathrm{TV}$ balls.} For every $d\ge 1$, $\|1_A\|_{\mathcal{R}}=+\infty$.
\item \textit{Naive smoothings can still have infinite $\mathcal{R}\mathrm{TV}$ in $d>1$.} Ramp and sigmoidal split-wise smoothings retain infinite $\mathcal{R}\mathrm{TV}$ under mild conditions (in particular, multiple split directions / depth $\ge 2$).
\item \textit{Gaussian smoothing yields finite $\mathcal{R}\mathrm{TV}$ but is dimension-dependent.} Convolution with an isotropic Gaussian produces $\|f_\sigma\|_{\mathcal{R}}<\infty$ with an explicit bound that scales exponentially in $d$.
\item \textit{Classification via thresholding admits finite-$\mathcal{R}\mathrm{TV}$ exact recovery.} We construct a smooth barrier score $S_B$ with $\{S_B\ge 1\}=B$ and prove (i) an $L_1(P)$ calibration bound under a tube-mass condition and (ii) an explicit $\mathcal{R}\mathrm{TV}$ upper bound in terms of the box face measures.
\item \textit{Experiments.} On synthetic unions of rectangles, we empirically trace the calibration--complexity frontier and illustrate how post hoc threshold tuning changes the induced decision set without changing the underlying learned score.
\end{enumerate}

So, if one only cares about the \emph{thresholded classifier}, shallow nets can represent trees easily. If one cares about learning a score that is both \emph{interpretable} (low $\mathcal{R}\mathrm{TV}$, gradients aligned to splits) and \emph{close} to the symbolic model, a quantitative trade-off emerges. This clarifies what exactly is in tension and why thresholding alone can mask feature misalignment.

%% file: related_work.tex
\section{Related Work}\label{sec:related}

\paragraph{The accuracy--interpretability debate.}
Early empirical work reported an apparent performance gap between transparent models (linear regressions, GAMs, decision trees) and deep neural nets \citep{doshivelez2017}. More recent benchmarks and case studies nuance this picture: in some regimes, carefully tuned interpretable models can rival black-box baselines, and the relationship between interpretability and predictive performance can be non-monotonic \citep{lovo2025tacklingaccuracyinterpretabilitytradeoffhierarchy,atrey2025demystifying}. In high-stakes settings, some authors argue that the presumed trade-off is overstated and advocate for interpretable-by-design predictors \citep{rudin2019}. Our analysis sharpens this discussion by identifying a dimension-dependent regime---captured through a geometric complexity measure---in which high accuracy necessarily coincides with large complexity.

\paragraph{Complexity measures for neural functions.}
Generalisation guarantees for neural networks are typically phrased in weight-space norms, e.g., the path-norm \citep{neyshabur2015pathsgdpathnormalizedoptimizationdeep}, products of spectral norms \citep{bartlett2017spectrallynormalizedmarginboundsneural}, or Neural Tangent Kernel radii \citep{jacot2020neuraltangentkernelconvergence}. While these quantities can correlate with test error (and sometimes robustness), their connection to \emph{geometric} properties of decision boundaries is indirect. A complementary viewpoint is provided by the \emph{Radon bounded-variation} space $\rbv{\Omega}$ ($\Omega \subseteq \reals^d$), defined via bounded Radon-domain total variation ($\mathcal{R}\mathrm{TV}$) \citep{savarese19a}. Representer theorems show that shallow ReLU networks trained with weight decay admit $\rbv{\Omega}$ solutions \citep{Parhi2020BanachSR}, and subsequent work show connection to reproducing kernel Banach spaces (RKBS) and related structural/approximation results \citep{Ongie2020A,mao2024approximationratesshallowreluk,kumar2024mirrordescentreproducingkernel}. Building on this framework, we analyze $\mathcal{R}\mathrm{TV}$ for both smoothed and hard decision-tree limits, highlighting regimes where symbolic transparency clashes with bounded-complexity function classes.

\paragraph{Approximation behaviour of shallow \emph{versus} deep networks.}
Depth-2 networks are universal approximators of continuous functions on bounded domains under mild conditions on the activation \citep{Cybenko1989ApproximationBS,HORNIK1989359,FUNAHASHI1989183}, but \citet{kumar2025gapgaussianrkhsneural} have shown that there exists Gaussian kernel machines that lie outside the representation scheme on unbounded domain. Classical results quantify the curse of dimensionality for single-hidden-layer approximation of smooth targets, with width scaling like $\epsilon^{-d/n}$ for $C^n$ functions \citep{Pinkus_1999,mhaskar}, and refinements for ReLU architectures clarifying how depth can mitigate this dependence \citep{pmlr-v75-yarotsky18a}. For discontinuous or sharply varying targets, complexity can be dominated by boundary geometry; in our setting, the $\mathcal{R}\mathrm{TV}$ needed to represent (or closely calibrate) tree-style boundaries can scale rapidly with dimension, so bounded-complexity shallow predictors may require large width or additional depth to match tree-level accuracy. Depth-separation results provide a complementary perspective: certain functions computable by moderately deep networks provably require exponentially many units at smaller depth \citep{telgarsky2016benefits,eldan2016power}.

\paragraph{Smooth decision trees and differentiable forests.}
Gradient-based tree learning has grown well beyond early neural decision forest constructions \citep{Kontschieder_2015_ICCV,frosst2017distillingneuralnetworksoft}. Neural Oblivious Decision Ensembles (NODE) combine stacks of feature-shared oblivious trees with end-to-end training and are competitive on tabular benchmarks \citep{popov2019node}. Adaptive Neural Trees learn both topology and predictors \citep{tanno2019ant}, while the Tree Ensemble Layer integrates a soft forest into a deep network and trains jointly with upstream representations \citep{hazimeh2020treeensemblelayer}. More recently, GRANDE optimizes large differentiable forests with Adam and reports strong empirical performance \citep{grande2023gradient}. These works show that back-prop can fit high-capacity tree-like models; our $\mathcal{R}\mathrm{TV}$ calculations provide a complementary explanation for why smoothing alone may not avoid dimension-driven complexity in regimes where sharp, axis-aligned boundaries dominate.

%% file: problem_setup.tex
\section{Problem Setup}\label{sec:setup}
\paragraph{Data space.}
We work on the ambient Euclidean space \(\cX = \mathbb{R}^{d}\)
with \(d\ge 1\).  We denote a datapoint by \(\x\in\cX\) and scalar values by $y,z,\omega \in \reals$.

\paragraph{Axis-aligned decision trees.}
A depth--$D$ axis-aligned decision tree consists of internal nodes $t$,
each testing a single coordinate $x_{j_t}$ against a threshold $\theta_t \in \mathbb{R}$.
Equivalently, each test can be written as $\mathbf{1}\{w_t^\top x + b_t > 0\}$
with $w_t \in \{\pm e_j\}$ and $b_t \in \mathbb{R}$.
Traversing the tree from root to leaf evaluates the classifier in $O(D)$ tests.

Each leaf $\ell$ corresponds to an axis-aligned cell (a box, possibly unbounded)
$B_\ell = \prod_{j=1}^d [\ell_{\ell,j},u_{\ell,j}]$, obtained by intersecting the
halfspace constraints along its root-to-leaf path. Let $\mathcal{L}_+$ denote
the leaves labeled $1$. The induced classifier is
\[
  f_{\sf DT}(x) \;=\; \mathbf{1}\{x\in A\},
  \qquad
  A \;=\; \bigcup_{\ell\in\mathcal{L}_+} B_\ell ,
\]
and (up to boundaries) this union is disjoint, so $f_{\sf DT}=\sum_{\ell\in\mathcal{L}_+}\mathbf{1}_{B_\ell}$ a.e.
Its decision boundary is composed of pieces of coordinate-aligned $(d-1)$-hyperplanes.
However, the jump discontinuities place \(f_{\sf{DT}}\) outside smooth
function classes such as reproducing kernel Hilbert spaces (RKHS) or the space of 2-layered infinite width ReLU neural networks,
giving \(\|f_{\sf{DT}}\|_{\mathcal{R}}=\infty\) (see \secref{sec:hard}).
 For analytic control we therefore introduce smooth surrogates
that soften each split while preserving this leaf/box structure (see \figref{fig:smoothing}).


\paragraph{Smoothed decision trees.}
All surrogates keep the same \((\bw_i,b_i)\) and depth \(D\); they
differ only in how the sign test \({1}\{\bw_i^{\!\top}\x+b_i>0\}\)
is replaced.

\emph{Piecewise-linear ramp smoothing.}
For a margin width \(\epsilon>0\) define
\begin{align}
  \rho_\epsilon(z)=
  \begin{cases}
    0, & z\le-\tfrac{\epsilon}{2},\\
    \frac{z}{\epsilon}+\frac{1}{2}, & |z|<\tfrac{\epsilon}{2},\\ \label{eq:ramp}
    1, & z\ge \tfrac{\epsilon}{2}.
  \end{cases}  
\end{align}
The ramp-smoothed tree is
\begin{align}
  f_{\sf{DT},\epsilon}(\x)
  \;=\;
  \prod_{i=1}^{D} \rho_\epsilon\!\bigl(\bw_i^{\!\top}\x+b_i\bigr).  \label{eq:rampdt}
\end{align}
It coincides with \(f_{\sf{DT}}\) outside width-\(\epsilon\) slabs and
converges to the hard tree as \(\epsilon\!\to\!0\).

\begin{figure}[t]
  \hspace{-12mm}
  \centering
  \begin{minipage}[t]{0.28\linewidth}
    \centering
    \includegraphics[width=\linewidth]{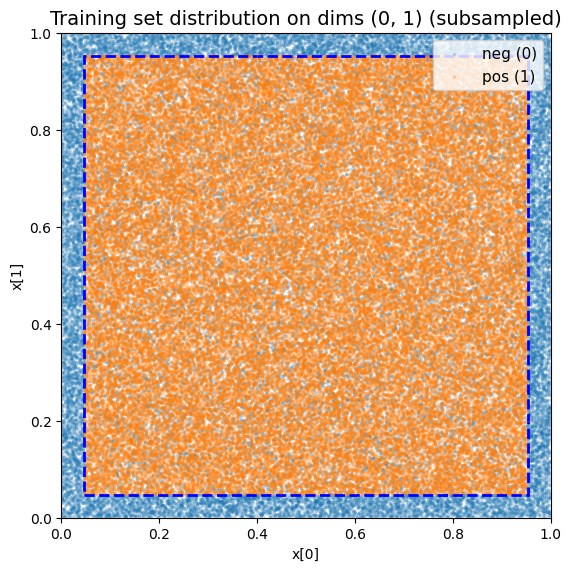}
    \vspace{0.3em}
    {\small\textbf{(a)} Training data (slice).}
  \end{minipage}
  \hspace{2mm}
  \begin{minipage}[t]{0.32\linewidth}
    \centering
    \includegraphics[width=\linewidth]{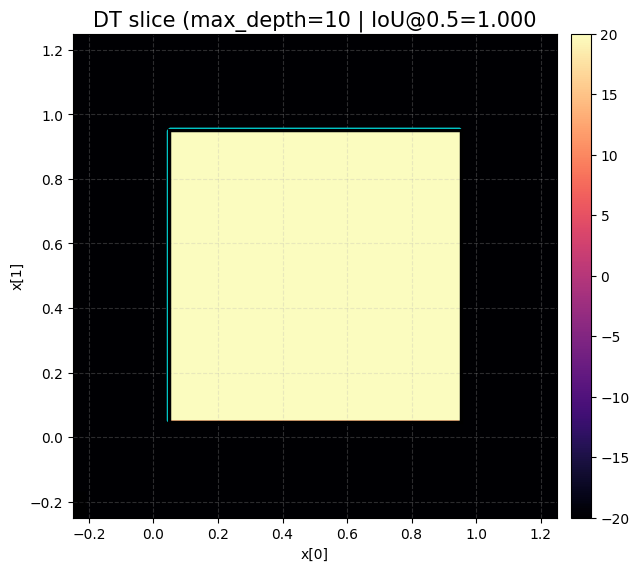}
    \vspace{0.3em}
    {\small\textbf{(b)} Decision tree boundary (slice).}
  \end{minipage}
  \begin{minipage}[t]{0.43\linewidth}
    \hspace*{4mm}
    \centering
    \includegraphics[width=\linewidth]{rtvvswidth.png}
    \vspace{0.3em}
    {\small\textbf{(c)} Width-$\mathcal{R}\mathrm{TV}$ frontier.}
  \end{minipage}
  \caption{\textbf{Box-classification dataset and decision-tree baseline.}
  \textbf{(a)} Labeled samples for the synthetic box task, visualized on the $(x_0,x_1)$ slice
  with $x_2=x_3=x_4=0.5$. Points are uniformly sampled from $[0,1]^5$ and colored by box membership;
  the dashed rectangle denotes the ground-truth boundary on the slice.
  \textbf{(b)} An axis-aligned decision tree fit on the same data, visualized via its induced slice boundary. \textbf{(c)} shows how the $\mathcal{R}\mathrm{TV}$ of shallow ReLU networks grows with width on the same box task, highlighting the trade-off between approximation quality and Radon total variation. For depth-1 ReLU networks trained with Adam (with weight decay) on the $d=10$ box classification task, where $X \sim \mathrm{Unif}([0,1]^{10})$ and $y = \mathbf{1}\{x \text{ lies in a centered axis-aligned box}\}$ (roughly $50\%$ positives). Curves correspond to test raw-MSE targets $0.20, 0.25, 0.30,$ and $0.40$; for each width $m\in\{8,16,32,64,128,256,512\}$ we plot the mean $\mathcal{R}\mathrm{TV}(W,a)=\tfrac12(\|W\|_F^2+\|a\|_2^2)$ at the first epoch when the test raw MSE drops below the target, averaged over multiple random initializations.
  }
  \label{fig:dt-training-data}
\end{figure}

\emph{Sigmoidal (logistic) smoothing.}
With temperature \(\gamma>0\) let
\(\sigma_\gamma(z)=\bigl(1+e^{-z/\gamma}\bigr)^{-1}\).
The model
\begin{align}
  f_{\sf{DT},\gamma}(\x)
  \;=\;
  \prod_{i=1}^{D} \sigma_\gamma\!\bigl(\bw_i^{\!\top}\x+b_i\bigr)  \label{eq:sig}
\end{align}
is infinitely differentiable; its transition width is \(O(\gamma)\) and
its spectrum decays faster than any polynomial in frequency, faster than the ramp yet
slower than the Gaussian surrogate below.

\emph{Gaussian smoothing.}
Global diffusion is obtained by convolving the hard tree with an
isotropic Gaussian kernel \(G_\sigma(z)=
 (2\pi\sigma^{2})^{-d/2}\exp(-\|z\|^{2}/(2\sigma^{2}))\):
 \begin{align}
  f_\sigma(\x)
  \;=\;
  \int_{\mathbb{R}^{d}}
    f_{\sf{DT}}(\y)\,G_\sigma(\x-\y)\,d\y.   \label{eq:gauss}
 \end{align}
This surrogate is \(C^\infty\) with Fourier transform
\(\widehat{f_\sigma}(\xi)=e^{-\sigma^{2}\|\xi\|^{2}/2}\widehat{f_{\sf{DT}}}(\xi)\),
implying exponential spectral decay.  Unlike the ramp or sigmoid
constructions—which preserve the separable, axis-aligned product
structure -- Gaussian convolution couples all coordinates, spreading the
effect of each split over a neighbourhood of radius~\(\sigma\).

Each smoothing scheme recovers \(f_{\sf{DT}}\) in the
limit \(\epsilon,\gamma,\sigma\to 0\), but exhibits markedly different
regularity and spectral behaviour; these differences will be central to
our subsequent analysis of their Radon total-variation norm.

\paragraph{Radon transform.}
For \(f\in L^{1}(\mathbb{R}^{d})\) the Radon transform is
\[
  \mathcal{R} f(\beta,t)
  :=\!
  \int_{\{x:\,\beta^{\!\top}x=t\}}
      f(x)\,ds(x),
  \qquad
  (\beta,t)\in\mathbb{S}^{d-1}\times\mathbb{R},
\]
where \(ds\) denotes the \((d-1)\)-Lebesgue measure on the hyperplane.
With the unitary Fourier convention
\(
  \widehat{f}(\xi)=(2\pi)^{-d/2}\int f(x)e^{-i\xi^{\top}x}\,dx
\)
the \emph{Fourier–slice theorem} (see, e.g.,
\cite{kak_slaney}) gives
\begin{equation}\label{eq: slice}
  \mathcal{R} f(\beta,t)
  \;=\;
  (2\pi)^{\tfrac{1-d}{2}}
  \int_{\mathbb{R}}
      e^{i\omega t}\,
      \widehat{f}(\omega\beta)\,d\omega .
\end{equation}

\paragraph{Second-order Radon bounded-variation space.}
Following \citet{Ongie2020A,Parhi2020BanachSR} we define
\[
  \rbv{\mathbb{R}^{d}}
  :=\bigl\{
      f\in L^{\infty,1}(\mathbb{R}^{d})
      :\|f\|_{\mathcal{R}}<\infty
    \bigr\},
\]
where
\begin{equation}\label{eq: RNormDef}
  \|f\|_{\mathcal{R}}
  :=c_{d}\,
     \bigl\|
       \partial_{t}^{2}\Lambda^{d-1}\mathcal{R} f
     \bigr\|_{\mathcal{M}(\mathbb{S}^{d-1}\times\mathbb{R})},
  \qquad
  c_{d}^{-1}=2(2\pi)^{d-1},
\end{equation}
and
\(\Lambda^{d-1}=(-\partial_{t}^{2})^{(d-1)/2}\) is the 1-D “ramp
filter’’ operator.\footnote{For odd \(d\) the fractional power is
interpreted via Fourier multipliers; all derivatives are taken in the
sense of tempered distributions.}
The norm \(\|\cdot\|_{\mathcal{R}}\) coincides with the minimum-width,
infinite-neuron ReLU network norm introduced by
\citet{Ongie2020A}, and measures the second-order total variation of
\(\mathcal{R} \curlybracket{f}\) across all projection directions. In this work we denote this seminorm by $\|\cdot\|_{\mathcal R}$.

Using the discussion above, it turns out the computation of $\|\cdot\|_{\mathcal R}$ in the one-dimensional setting can be simplified. This has been formally proven in \cite{savarese19a} as follows:

\begin{theorem}[Theorem 3.1~\cite{savarese19a}]
    For any function $f: \reals \to \reals$, we have:
    \begin{align*}
        \rtv{f} = \max\paren{\int_{-\infty}^{\infty} |f''(x)| d x, |f'(\infty) + f'(-\infty)|} \le \int_{-\infty}^{\infty} |f''(x)| d x + 2 \inf_{x} |f'(x)|
    \end{align*}
\end{theorem}
In higher dimension, we follow a three-step procedure to compute $\|\cdot\|_{\mathcal R}$. 
\subsection{A universal three-step recipe for \texorpdfstring{$\|f\|_{\mathcal{R}}$}{|f|R}}
\eqnref{eq: slice} and linearity yield the following
computational pattern, which we employ in this work.

\begin{enumerate}
\item \textbf{Fourier transform.}
      Obtain (exactly or up to an explicit bound) the Fourier transform
      \(\widehat{f}(\xi)\).

\item \textbf{One-dimensional Radon slice.}\\
      Substitute \(\widehat{f}(\omega\beta)\) into \eqref{eq: slice}
      and differentiate once more in \(t\):
      $$
        \partial_{t}^{d+1}\mathcal{R} f(\beta,t)
        =(2\pi)^{(1-d)/2}
        \int_{\mathbb{R}}
            (i\omega)^{d+1}
            e^{i\omega t}\,
            \widehat{f}(\omega\beta)\,d\omega .
      $$

\item \textbf{\(L^{1}\)-norm of the \(\,(d\!+\!1)\)st derivative.}\\
      Integrate the absolute value over
      \(t\in\mathbb{R}\) and \(\beta\in\mathbb{S}^{d-1}\),
      applying Fubini/Tonelli and any required bounds on
      \(\widehat{f}\) to obtain
      \(
        \|f\|_{\mathcal{R}}
      \).
\end{enumerate}

\noindent
Intuitively, Step 1 encodes geometric information in frequency space,
Step 2 converts that information into directional line integrals, and
Step 3 aggregates the variation of these integrals to yield the Radon
BV norm.  The recipe is agnostic to the specific form of \(f\); it
applies verbatim to the hard tree, ramp-smoothed, sigmoid-smoothed and
Gaussian-smoothed models introduced earlier, differing only in the
bounds used for \(\widehat{f}\).

This formalises the procedure implicit in previous discussions: each
$||\cdot||_\cR$-norm computation reduces to a Fourier bound followed by a one-D
integration in the projection variable~\(t\).

\begin{figure}[t!]
  \centering
  \begin{minipage}[t]{1.1\linewidth}
    \hspace{-20mm}
    \centering
    \includegraphics[width=\linewidth]{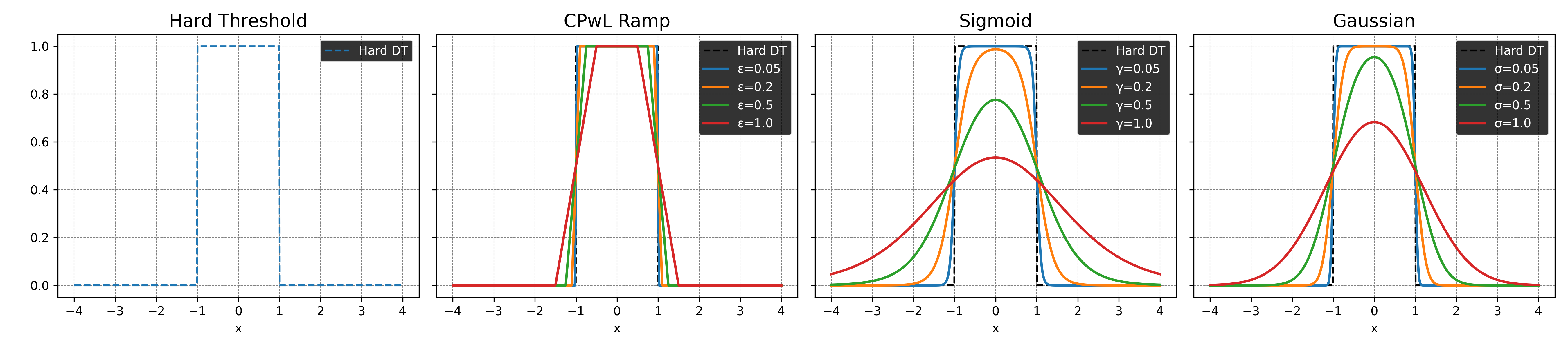}
  \end{minipage}\hfill
  \caption{Plot visualizes how ramp, sigmoidal, and Gaussian smoothings soften a one-dimensional hard threshold.}
  \label{fig:smoothing}
\end{figure}




%% file: hard-threshold.tex
\section{Approximation of Hard-Threshold Decision Trees via Shallow Networks}\label{sec:hard}

In \secref{sec:setup} we introduced the Radon total‐variation norm $\|\cdot\|_{\mathcal R}$ and its explicit form in \eqnref{eq: RNormDef}.  We now establish that this norm is \emph{unbounded} for hard-threshold decision trees (\thmref{thm:hard}).  We start with the one–dimensional setting, where a decision tree reduces to a step function taking values in $\curly{0,1}$.

Consider the step function in single dimension denoted as $f_{\sf{step}}: \reals \to \reals$, defined as
\begin{align*}
    f_{\sf{step}}(x) = \sum_{i=1}^{n} c_i\cdot {1}\curlybracket{x \in (z_i, z_{i+1})}
\end{align*}
for given set of scalars $-\infty < z_1 \le z_2 \le \ldots \le z_N < \infty$.

\begin{lemma}\label{lem:step}
    $\rtv{f_{\sf{step}}}$ is unbounded.\vspace{-3mm}
\end{lemma}
\begin{proof}[Proof outline]
For \(d=1\),
\(\|f\|_{\mathcal R}= \int_{\mathbb R}\!|f''(x)|\,dx\).
Each jump at \(x=z_i\) yields
\(f''=c_i\,\delta'_{z_i}\), where
\(\delta'_{z_i}\) is a \emph{dipole distribution}
(\(\langle\delta'_{z_i},\varphi\rangle=-\varphi'(z_i)\)).
Approximating
\(\delta'_{z_i}\) by the mollifier
\(\delta'_{\varepsilon}(x-z_i)=\varepsilon^{-2}\psi'((x-z_i)/\varepsilon)\)
gives
\(\|\delta'_{z_i}\|_{L^{1}}
      =\lim_{\varepsilon\to 0}\varepsilon^{-2}\!\int|\psi'(u)|\,du
      =\infty\).
Hence\vspace{-2mm}
\[
  \int_{\mathbb R}|f''(x)|\,dx
  =\sum_{i=1}^{n}|c_i|\,
    \|\delta'_{z_i}\|_{L^{1}}
  =\infty,
\]
so the $\mathcal{R}\mathrm{TV}$ seminorm diverges.
\end{proof}
The divergence stems from the $(d\!+\!1)$-st derivative of an indicator, which contains derivatives of the Dirac delta distribution.  These derivatives have infinite total‐variation (equivalently, $\ell_1$) norm, so no amount of averaging can regularise them.  Because every axis-aligned decision tree contains a one–dimensional slice exhibiting the same pathology, the $\mathcal{R}\mathrm{TV}$ remains unbounded in any ambient dimension.  The formal statement is given in \thmref{thm:hard}; its proof is deferred to the supplemental materials.

\begin{theorem}\label{thm:hard}
    Consider the decision tree $f_{\sf{DT}}: \reals^d \to \reals$ such that $f_{\sf{DT}}(\x) := {1}\curly{\x \in A}$ for an axes aligned compact subset $A \subset \reals^d$. Then, 
    $\rtv{f_{\sf{DT}}}$ of the decision tree as defined is unbounded.
\end{theorem}
\begin{proof}[Proof outline]
First observe that the Fourier transform of the indicator over the axis-aligned box $A$ is  
\begin{align*}
    \hat{f}_{\sf{DT}}(\bm{\xi})
      \;=\;
      (2\pi)^{-d/2}
      \int_{A} e^{- \mathrm{i} \bm{\xi}^\top \x}\,d\x .
\end{align*}
Consequently, for odd dimension $d$ the Radon-TV norm becomes
\[
\rtv{f_{\sf{DT}}}
  =
  c_d(2\pi)^{-(d-1)/2}
  \int_{\mathbb{S}^{d-1}}
  \int_{\mathbb{R}}
    \Bigl|
      \int_A
        \delta^{(d+1)}\!\bigl(t-\bm{\beta}^{\top}\x\bigr)\,d\x
    \Bigr|
  \,dt\,d\bm{\beta}.
\]

\vspace{-2mm} Fix the direction $\bm{\beta}_0=\mathbf{e}_1\in\mathbb{S}^{d-1}$ and let  
$\mathcal{B}(\bm{\beta}_0,\epsilon)\subseteq\mathbb{S}^{d-1}$ be the spherical cap of radius $\epsilon>0$.  
For each $\bm{\beta}\in\mathbb{S}^{d-1}$ define  
\[
  g_{\bm{\beta}}(u)
    := \int_A \delta\!\bigl(u-\bm{\beta}^{\top}\x\bigr)\,d\x .
\]
\vspace{-1mm} By the co-area formula \citep{Mattila_1995}, this can be rewritten as  
\[
  g_{\bm{\beta}}(u)
    = \int_{A\cap\{\x:\bm{\beta}^{\top}\x=u\}}
        \frac{1}{\|\bm{\beta}\|}\,d\sigma(\x)
    \;=\;
    \text{Vol}_{d-1}\!\bigl(A\cap\{\x:\bm{\beta}^{\top}\x=u\}\bigr),
\]
where $d\sigma$ denotes the $(d-1)$-dimensional Hausdorff measure on the hyperplane \(\{\x:\bm{\beta}^{\top}\x=u\}\).

For $\bm{\beta}_0=\mathbf{e}_1$ this is the step function  
\[
g_{\mathbf{e}_1}(u)=
\begin{cases}
0, & u\notin[a_1,b_1],\\[3pt]
\text{Vol}_{d-1}\bigl([a_2,b_2]\times\cdots\times[a_d,b_d]\bigr), & u\in(a_1,b_1).
\end{cases}
\]
Hence $g_{\mathbf{e}_1}$ has sharp jumps at $u=a_1$ and $u=b_1$. Perturbing to $\bm{\beta}=\bm{\beta}_0+\Delta$ with $\|\Delta\|\le\epsilon$ smooths these jumps over an $O(\epsilon)$ interval, but the slope grows rapidly.  In particular, for sufficiently small $\epsilon>0$ there exists \(u_\epsilon\approx a_1\) such that
\[
  \bigl|g_{\bm{\beta}_0+\Delta}^{(d+1)}(u_\epsilon)\bigr|
  \;\approx\;
  \frac{C}{\epsilon^{d+1}}
\]
for some constant \(C>0\).
Integrating over a window of width $\epsilon$ yields  
\[
  \int_{\mathbb{R}}
    \bigl|g_{\bm{\beta}_0+\Delta}^{(d+1)}(u)\bigr|
  \,du
  \;\ge\;
  \frac{C'}{\epsilon^{d}},\,\, \textnormal{with } C'>0.
\]

Now integrate this lower bound over the cap $\mathcal{B}(\bm{\beta}_0,\epsilon)$, whose surface measure scales as $\epsilon^{d-1}$:
\[
  \int_{\mathcal{B}(\bm{\beta}_0,\epsilon)}
    \int_{\mathbb{R}}
      \bigl|g_{\bm{\beta}}^{(d+1)}(u)\bigr|
    \,du
  \,d\bm{\beta}
  \;\ge\;
  \epsilon^{d-1}\,\frac{C'}{\epsilon^{d}}
  \;=\;
  \frac{C'}{\epsilon}.
\]
Because this bound holds for every sufficiently small $\epsilon$, letting $\epsilon\to 0$ forces $\rtv{f_{\sf{DT}}}=\infty$.

The same argument applies to any canonical direction, completing the proof for general axis-aligned decision trees.
\end{proof}

\begin{remark}[Implication for shallow networks]
Any single‐hidden‐layer ReLU network whose weights are chosen so that the
induced function has bounded Radon‐TV cannot approximate a hard-threshold
decision tree to arbitrary accuracy. Additional depth or unbounded weight growth is necessary.
\end{remark}


%% file: smoothing.tex
\section{Approximation of Smoothed Decision Trees via Shallow Networks}\label{sec:smooth}

Section~\ref{sec:hard} showed that the Radon total-variation ($\mathcal{R}\mathrm{TV}$) norm is infinite for hard-threshold decision trees. At first sight this is puzzling: the universal approximation theorem guarantees that even a \emph{shallow} neural network can approximate any continuous function to arbitrary accuracy. The catch is the discontinuity---no finite $\mathcal{R}\mathrm{TV}$ ball, regardless of width, can capture a step. We therefore ask whether \emph{smoothing} the tree reduces the norm.

\subsection{Piecewise-linear (ramp) smoothing}

Equation~\eqref{eq:rampdt} replaces each hard split by a centred ramp, producing a continuous piecewise-linear function.  The behaviour differs sharply between one- and higher-dimensional domains.

\paragraph{The one--dimensional case.}
For $d=1$, each factor $\rho_\epsilon(w_i x+b_i)$ has breakpoints at
\[
t_i^\pm := \frac{-b_i \pm \epsilon/2}{w_i},\qquad i=1,\dots,D,
\]
(with the understanding that $t_i^-<t_i^+$ after ordering).
Let $s_1<\cdots<s_M$ be the distinct sorted points in $\{t_i^-,t_i^+\}_{i=1}^D$
so that $M\le 2D$.
Assume the transition intervals $(t_i^-,t_i^+)$ are pairwise disjoint, so that
on each open interval $(s_k,s_{k+1})$ at most one factor is in its linear regime.
Then $f_{{\sf DT},\epsilon}(x)=\prod_{i=1}^D \rho_\epsilon(w_i x+b_i)$ is continuous piecewise-linear,
hence affine on each $(s_k,s_{k+1})$ with slope $m_k$.
Writing $c_0:=f_{{\sf DT},\epsilon}(x_0)$ for some $x_0<s_1$ and
$c_k:=m_k-m_{k-1}$ for $k=1,\dots,M$, we obtain
\[
f_{{\sf DT},\epsilon}(x)=c_0+\sum_{k=1}^{M} c_k\,[x-s_k]_+.
\]
Hence a single-hidden-layer ReLU network with $M$ units represents $f_{{\sf DT},\epsilon}$
exactly on~$\mathbb R$; in particular $f_{{\sf DT},\epsilon}\in \rbv{\mathbb R}$.

\paragraph{The multi-dimensional case.}
When $d>1$ the situation changes.
For continuous piecewise-linear (CPwL) surrogates with compact support,
the $\mathcal{R}\mathrm{TV}$ norm is typically infinite as soon as the function has kinks in more than
one direction. In particular, \citet{Ongie2020A} show that a generic compactly supported CPwL
function has $\|f\|_{\mathcal{R}}=+\infty$ whenever its boundary normals are not all parallel.
Thus, any ramp-smoothed construction that induces split boundaries in two distinct directions
(e.g., uses two different coordinates) will generally have infinite $\mathcal{R}$-norm; the main
exception is the essentially one-dimensional ridge case where all split normals are parallel, which
reduces to the $d=1$ setting.


\begin{proposition}[Proposition 5-(a) of \citealp{Ongie2020A}]\label{prop:parallel}
Suppose \( f : \mathbb{R}^d \to \mathbb{R} \) is a continuous piecewise linear function with compact support. If at least one of the boundary normals is not parallel with every other boundary normal.  Then \( f \) has infinite \(\mathcal{R}\)-norm. \vspace{-1mm}
\end{proposition}

In a ramp-smoothed tree the hyperplanes
\(
  \{\x\!:\,(\w_i^{\top}\x+b_i)/\epsilon+{\tfrac12}=0\}
\)
intersect the planes \(\{(\bw, b)=(\mathbf 0,0)\}\) and \(\{(\bw, b)=(\mathbf 0,1)\}\) along $(d-1)$-dimensional faces, so condition (a) is met whenever the ambient dimension satisfies \(d>1\).  Moreover, constructing a continuous approximation to a hard threshold requires at least four ramps (\(D\ge 4\)), ensuring the hypothesis of Proposition~\ref{prop:parallel}.  Consequently, for \(d>1\) and realistic depths the Radon norm of a piecewise-linear smoothed tree is still infinite.

\subsection{Smoothed Decision Trees}\label{subsec:smoother}

We now investigate two differentiable smoothing schemes that exhibit markedly different spectral decay: the \emph{sigmoidal} and the \emph{Gaussian} smoothings.  For the sigmoidal case we show that the Radon norm remains infinite (\thmref{thm:sigmoid}), whereas Gaussian smoothing yields a finite norm whose magnitude depends explicitly on the ambient dimension (\thmref{thm:gaussian}).\footnote{For one–dimensional signals certain smoothings do produce a bounded Radon norm; the argument is analogous to the $d=1$ ramp analysis and is omitted for brevity.}

\paragraph{Sigmoidal smoothing.}
For any $\gamma>0$ and shift $b\in\mathbb R$ the Fourier transform of the scaled logistic
\[
   \sigma_\gamma(z)
     \;=\;
     \tfrac{1}{1+e^{-z/\gamma}}
     \;=\;
     \tfrac12+\tfrac12\tanh\!\bigl(z/(2\gamma)\bigr)
\]
is
\begin{equation}
    \widehat{\sigma_\gamma(\,\cdot+b)}(\omega)
       \;=\;
       e^{-\mathrm i\omega b}
       \Bigl[
         \frac{\pi}{2}\,\delta(\omega)
         +\frac{\mathrm i\,\gamma\pi}{\sinh(\pi\gamma\omega)}
       \Bigr],
    \qquad \omega\in\mathbb R.
    \label{eq:fsig}
\end{equation}
Using \eqnref{eq:fsig}, the Fourier transform of the sigmoidal \(D\)-split tree
\[
   f_{\sf{DT},\gamma}(\x)
     \;:=\;
     \prod_{i=1}^{D}\sigma_\gamma\!\bigl(\w_i^{\top}\x+b_i\bigr),
     \qquad
     \{\w_i\}_{i=1}^{D}\subset\mathbb R^{d},
\]
with orthonormal normals $\w_i$, is
\[
   \widehat{f_{\sf{DT},\gamma}}(\bm{\xi})
     \;=\;
     (2\pi)^{-(d-D)/2}\,
     \delta\!\bigl(P_{\mathcal S^{\perp}}\bm{\xi}\bigr)
     \prod_{i=1}^{D}
        e^{-\mathrm i\,b_i\eta_i}
        \Bigl[\tfrac{\pi}{2}\,\delta(\eta_i)
              +\tfrac{\mathrm i\,\gamma\pi}{\sinh(\pi\gamma\eta_i)}\Bigr],
     \quad
     \eta_i=\w_i^{\top}\bm{\xi},
\]
where $\mathcal S=\operatorname{span}\{\w_1,\dots,\w_D\}$.  Substituting this expression into the norm definition \eqnref{eq: RNormDef} yields the following theorem. We defer the proof to the supplemental materials.

\begin{theorem}\label{thm:sigmoid}
Fix a depth $D\ge 1$ and temperature $\gamma>0$, and let
\[
  f_{\sf{DT},\gamma}(\x)
     \;=\;
     \prod_{i=1}^{D}\sigma_\gamma\!\bigl(\w_i^{\top}\x+b_i\bigr),
     \qquad
     \{\w_i\}_{i=1}^{D}\ \text{orthonormal and distinct}.
\]
Then
\[
    \|f_{\sf{DT},\gamma}\|_{\mathcal R}
      =\frac{c_{d}}{\sqrt{2\pi}}\,
        (2\pi)^{-\frac{d-D}{2}}\,
        (\gamma\pi)^{D}
        \int_{\mathbb S^{d-1}}
        \int_{\mathbb R}
            |\omega|^{d+1}\,
            \mathbf 1_{\{\bm{\beta}\in\mathcal S\}}
            \prod_{i=1}^{D}
              \frac{1}
                   {\bigl|\sinh\bigl(\pi\gamma\lambda_i(\bm{\beta})\,\omega\bigr)\bigr|}
        \,d\omega\,d\sigma(\bm{\beta}),
\]
where $\lambda_i(\bm{\beta})=\w_i^{\!\top}\bm{\beta}$.  In particular, for every $D\ge 2$ the Radon norm diverges:
\(
   \|f_{\sf{DT},\gamma}\|_{\mathcal R}= \infty.
\)
\end{theorem}

    
Note that for $D = 1$, as noted in the case of linear ramp smoothing, it is straight-forward to note that $\|f_{\sf{DT},\gamma}\|_{\mathcal R}$ is bounded.

\paragraph{Gaussian smoothing.}
Convolving the hard tree with an isotropic Gaussian rapidly suppresses high-frequency components: in the Fourier domain
\(
\widehat{f_\sigma}(\bm{\xi})=\widehat{f_{\sf{DT}}}(\bm{\xi})\,e^{-\sigma^{2}\|\bm{\xi}\|^{2}/2}
\),
so every factor of \(\|\bm{\xi}\|\) in the Radon norm integrand is offset by Gaussian decay.  
We state the main theorem that quantifies the gain, with the proof deferred to the supplemental materials.

\begin{theorem}\label{thm:gaussian}
Fix any bandwidth \(\sigma>0\).  
For the Gaussian smoothed tree 
\[
f_\sigma(\x)=\int_{\mathbb R^{d}}\!f_{\sf{DT}}(\y)\,
            G_\sigma(\x-\y)\,d\y,
\qquad 
G_\sigma(z)=(2\pi\sigma^{2})^{-d/2}e^{-\|z\|^{2}/(2\sigma^{2})},
\]
the we show
$
\|f_\sigma\|_{\mathcal R}\;\le\;
 C\,d^{\,1/2}\;
   \Bigl(\tfrac{\sqrt{2}\,e}{\sigma}\Bigr)^{\!d}\;
   Vol(A),$ for 
$d\ge1,$
with a universal constant \(C\le1.2\).
\end{theorem}

%% file: threshold_class.tex
\FloatBarrier

\begin{figure}[t]
  \centering
  \begin{minipage}[t]{0.305\linewidth}
    \centering
    \includegraphics[width=\linewidth]{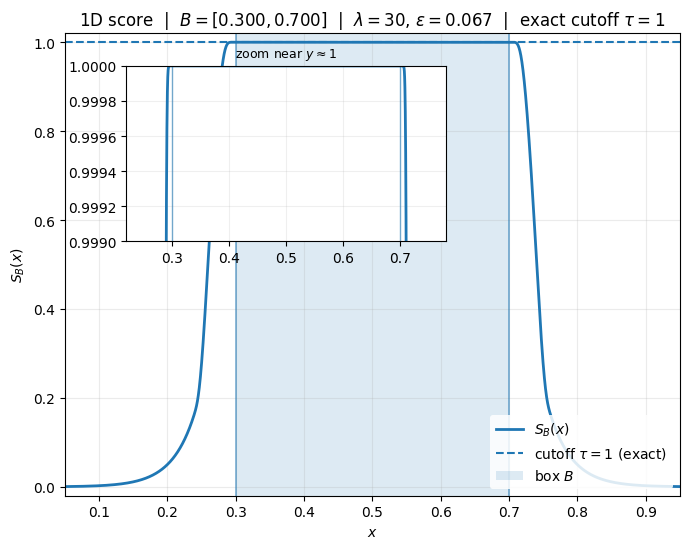}
    \vspace{0.3em}
    {\small\textbf{(a)} $d=1$.}
  \end{minipage}\hfill
  \begin{minipage}[t]{0.34\linewidth}
    \centering
    \includegraphics[width=\linewidth]{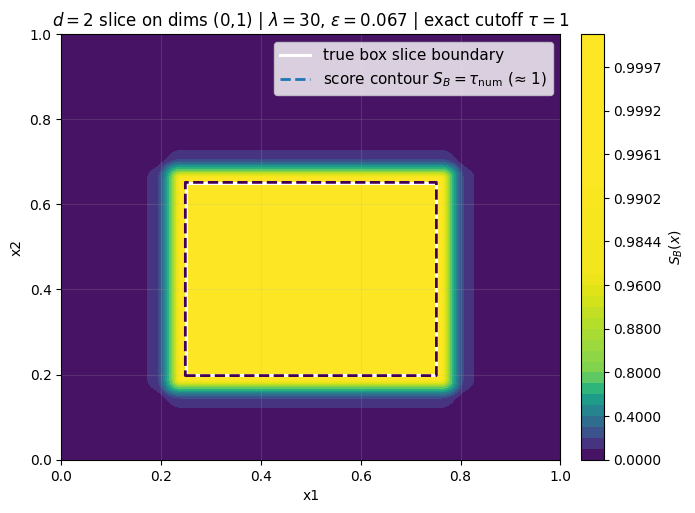}
    \vspace{0.3em}
    {\small\textbf{(b)} $d=2$.}
  \end{minipage}\hfill
  \begin{minipage}[t]{0.34\linewidth}
    \centering
    \includegraphics[width=\linewidth]{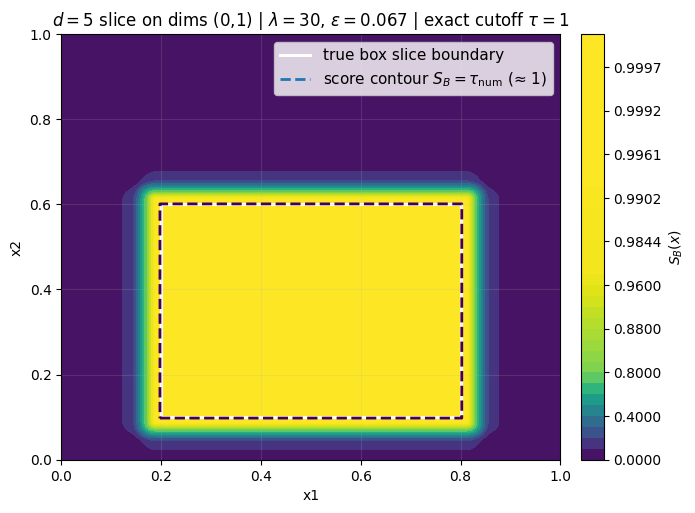}
    \vspace{0.3em}
    {\small\textbf{(c)} $d=5$ slice.}
  \end{minipage}
  \caption{\textbf{Smooth score that thresholds exactly to an axis-aligned box.}
We visualize the construction
$
S_B(x)=\prod_{j=1}^d
\vartheta_{\lambda,\varepsilon}(u_j-x_j)\,\vartheta_{\lambda,\varepsilon}(x_j-\ell_j)\in[0,1]
$
for a box $B=\prod_{j=1}^d[\ell_j,u_j]\subset[0,1]^d$, where
$
\vartheta_{\lambda,\varepsilon}(t)=(1-h_\varepsilon(t))\,e^{\lambda\min(t,0)}+h_\varepsilon(t),
\qquad
h_\varepsilon(t)=H\!\big((t+\varepsilon)/\varepsilon\big),
$
with $H\in C^\infty(\mathbb{R})$ a monotone cutoff satisfying $H(s)=0$ for $s\le 0$ and $H(s)=1$ for $s\ge 1$
(with vanishing endpoint derivatives), and $\varepsilon=c_0/\lambda$.
\textbf{(a)} 1D profile: $S_B(x)=1$ on $B$ and decreases smoothly outside.
\textbf{(b)} 2D heatmap of $S_B$ showing a plateau at $1$ on $B$ and smooth variation outside.
\textbf{(c)} 2D slice of the same score in $d=5$ (holding the remaining coordinates fixed inside $B$).
In all panels, the exact classification cutoff is $\tau=1$, i.e.,
$
\{x:\ S_B(x)\ge 1\}=B;
$
the solid white rectangle is the true box boundary on the displayed slice, and the dashed contour marks the numerical level set
$S_B(x)=1-10^{-12}$.}
  \label{fig:smooth_score_box}
\end{figure}

\section{Classification Problem is Easy}\label{sec:easy}

In the previous sections we showed that the hard tree indicator $1_A$ has infinite Radon total variation ($\mathcal{R}\mathrm{TV}$), and that several natural smooth surrogates can still have infinite $\mathcal{R}\mathrm{TV}$ in dimensions $d>1$. These results are most naturally interpreted as limitations for \emph{score learning}: approximating the discontinuous tree function (or even certain split-wise smoothings) while keeping $\mathcal{R}\mathrm{TV}$ bounded.

Here we show that the situation changes if we only care about \emph{classification after thresholding}. Concretely, for a decision set $A$ we ask for a smooth score $s:\mathbb{R}^d\to[0,1]$ such that $\{x:s(x)\ge \tau\}=A$ for a prescribed cutoff $\tau$. This requirement can be met by a score whose $\mathcal{R}\mathrm{TV}$ is finite and explicitly controlled, even when $1_A$ itself has infinite $\mathcal{R}\mathrm{TV}$. We focus on a single axis-aligned box (extensions to finite unions are straightforward).

\paragraph{Construction (single box).}
Fix $\lambda\ge 1$ and set $\varepsilon=c_0/\lambda$ for a constant $c_0>0$.
Let $H\in C^\infty(\mathbb{R})$ be nondecreasing with $H(s)=0$ for $s\le 0$, $H(s)=1$ for $s\ge 1$, and $H^{(m)}(0)=H^{(m)}(1)=0$ for $1\le m\le d+1$.
Define
\[
h_\varepsilon(t) \;:=\; H\!\left(\frac{t+\varepsilon}{\varepsilon}\right),
\qquad
\vartheta_{\lambda,\varepsilon}(t)
\;:=\; (1-h_\varepsilon(t))\,e^{\lambda t} + h_\varepsilon(t).
\]
Then $\vartheta_{\lambda,\varepsilon}(t)=1$ for $t\ge 0$ and $\vartheta_{\lambda,\varepsilon}(t)=e^{\lambda t}$ for $t\le -\varepsilon$, with a $C^{d+1}$ transition on $[-\varepsilon,0]$.

For a box $B=\prod_{j=1}^d[\ell_j,u_j]$, define the score
\[
S_B(x)
\;:=\;
\prod_{j=1}^d
\vartheta_{\lambda,\varepsilon}(u_j-x_j)\;
\vartheta_{\lambda,\varepsilon}(x_j-\ell_j)
\;\in\;[0,1].
\]
Inside $B$ all factors equal $1$, while outside $B$ at least one factor is strictly smaller than $1$. We visualize $S_B$ in \figref{fig:smooth_score_box}.

\paragraph{Exact thresholding and calibration.}
We quantify both exact recovery of the decision set at a fixed cutoff and distributional calibration in $L_1(P)$ under a standard tube-mass condition near $\partial B$.

\begin{assumption}[Tube--mass near $\partial B$]\label{assump:tube}
There exist constants $C>0$ and $\beta>0$ such that for all sufficiently small $t>0$,
\[
T_B(t)\;:=\;\mathbb{P}(\mathrm{dist}(X,\partial B)\le t) \;\le\; C\,t^\beta,
\qquad X\sim P.
\]
\end{assumption}

\begin{lemma}[Single box: exact thresholding and $L_1(P)$ control]\label{lem:single-box-main}
Let $S_B$ be defined as above. Then $\{x:S_B(x)\ge 1\}=B$. Moreover, under Assumption~\ref{assump:tube}, for all $\lambda\ge 1$,
\[
\mathbb{E}\bigl[\,|S_B(X)-\mathbf{1}_B(X)|\,\bigr]
\;\le\; C_{d,\beta,c_0}\, \lambda^{-\beta},
\qquad X\sim P,
\]
where $C_{d,\beta,c_0}$ depends only on $(d,\beta,c_0)$ and the constants in Assumption~\ref{assump:tube}.
\end{lemma}

\paragraph{Finite $\mathcal{R}\mathrm{TV}$ with an explicit bound.}
Beyond exact recovery and calibration, the barrier score admits a finite and quantitative $\mathcal{R}\mathrm{TV}$ bound in terms of the face measures of $B$.

\begin{theorem}[$\mathcal{R}\mathrm{TV}$ upper bound for a single box]\label{thm:RTV-single-main}
For all $\lambda\ge 1$ and $\varepsilon=c_0/\lambda$,
\[
\|S_B\|_{\mathcal{R}}
\;\le\;
C_d
\sum_{r=1}^d
\lambda^{\,d+1-r}\;
\mathcal{H}^{d-r}\!\bigl(\Sigma_{d-r}(B)\bigr),
\]
where $\Sigma_{d-r}(B)$ is the union of $(d-r)$-dimensional faces of $B$, $\mathcal{H}^{d-r}$ is
$(d-r)$-dimensional Hausdorff measure, and $C_d$ depends only on $d$ (and the choice of $H$).
\end{theorem}

We defer the proof to Appendix~\ref{app:post-thres}. Taken together, Lemma~\ref{lem:single-box-main} and Theorem~\ref{thm:RTV-single-main} show that if we only care about the \emph{thresholded decision set}, then shallow models can represent axis-aligned box classifiers using smooth scores with controlled $\mathcal{R}\mathrm{TV}$. The difficulty identified earlier arises when one instead asks for a calibrated score that closely approximates the discontinuous indicator (or certain split-wise smoothings) while keeping $\mathcal{R}\mathrm{TV}$ small.

\section{Experiments}\label{sec:experiments}

\paragraph{Task and data.}
We consider a synthetic axis-aligned box classification task in $\mathbb{R}^d$.
For $d=5$, inputs $x$ are sampled i.i.d.\ uniformly from $[0,1]^d$
($100{,}000$ train, $20{,}000$ test), and labels are
\[
y \;=\; \mathbf{1}\!\left\{x \in [\ell,u]^d\right\},
\qquad
\ell = 0.0471381679,\ \ u = 0.9528618321,
\]
which leads to roughly balanced classes (box volume $\approx 0.59$).
We hold out $20{,}000$ points from the training set for validation.

\paragraph{Models, loss, and evaluation.}
We train single-hidden-layer ReLU MLPs (widths $W\in\{8,16,32,64,128,256\}$)
using an \emph{MSE loss} on the \emph{raw network output} $f_\theta(x)\in\mathbb{R}$
(i.e., a \emph{logit score}).
We report test-set Intersection-over-Union (IoU) under two thresholds:
\[
\text{(fixed)}\quad \tau=0 \qquad\text{and}\qquad
\text{(optimized)}\quad \tau^\star \in \arg\max_{\tau\in\mathcal{G}} \mathrm{IoU}_{\mathrm{val}}(\tau),
\]
where $\mathcal{G}$ is a uniform grid of $201$ thresholds spanning the range of validation logits.
In particular, \emph{IoU@0.5} in our plots corresponds to the level set $\{x:f_\theta(x)\ge 0\}$
(since $p=0.5$ corresponds to logit $0$), not to a probability threshold.
IoU is computed on the full $5$-D test set as
\[
\mathrm{IoU} \;=\; \frac{\mathrm{TP}}{\mathrm{TP}+\mathrm{FP}+\mathrm{FN}},
\]
so it depends only on the overlap between the predicted-positive region and the true box.

\paragraph{Visualization.}
For visualization we plot $2$-D slices on dims $(0,1)$: we fix $x_2=x_3=x_4=0.5$ and vary
$(x_0,x_1)$ on a $500\times 500$ grid over $[-0.25,1.25]^2$.
Heatmaps show the raw score $f_\theta(x)$.
Decision boundaries are drawn as contours $f_\theta(x)=\tau$ (cyan), and we optionally overlay
offset contours $\tau\pm 0.1$ (white dotted).
We also overlay iso-contours at $f_\theta(x)=0$ and $f_\theta(x)=1$ (pink/magenta) to indicate score scaling.
The ground-truth box boundary on the slice is drawn as a black dashed rectangle.
We annotate the slice-IoU only as a sanity check; all reported IoUs are computed in the full $5$-D space. See Figures~\ref{fig:nn-box-logits-a}--\ref{fig:nn-box-logits-c} and Figures~\ref{fig:nn-box-threshold-a}--\ref{fig:nn-box-threshold-c}.

\paragraph{Decision-tree baseline.}
Figure~\ref{fig:dt-training-data} shows the labeled training set (on the same $(0,1)$ slice)
and an axis-aligned decision tree fit used as a baseline in the experiments below.

\paragraph{RTV proxy.}
To track a simple complexity surrogate across widths, we compute the weight-based
Radon-total-variation proxy
$\mathcal{R}\mathrm{TV}=\tfrac12\big(\sum_i \lVert w_i\rVert_2^2+\sum_i a_i^2\big)$
from the trained network parameters (biases excluded).
In our runs, $\mathcal{R}\mathrm{TV}$ increases with width: \vspace{-2mm}
\[
W = [8, 16, 32,64,128,256]:\ [24.1, 39.6, 119.7, 189.2, 264.2, 374.8],
\]
highlighting that improved geometric fit can coincide with substantially larger variation/complexity.


\begin{figure}[h!]
  \centering
  \setlength{\tabcolsep}{4pt}
  \renewcommand{\arraystretch}{1.0}
  \begin{tabular}{c c c}
    & \scriptsize $W=8$ & \scriptsize $W=16$ \\
    \scriptsize M &
    \includegraphics[width=0.45\linewidth]{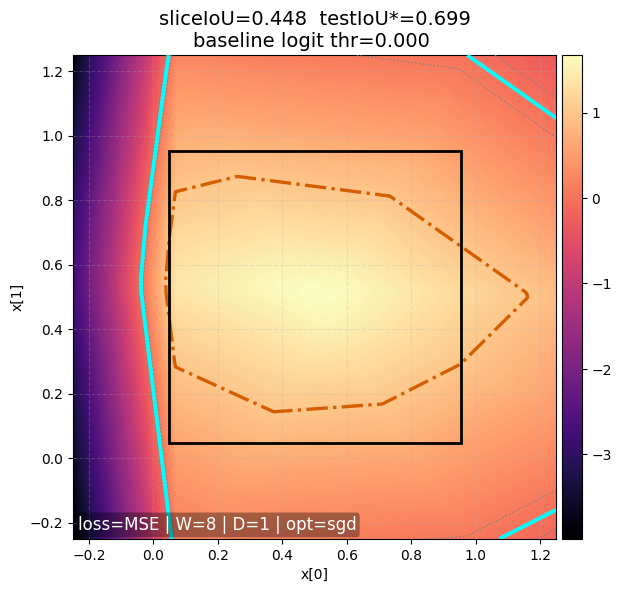} &
    \includegraphics[width=0.45\linewidth]{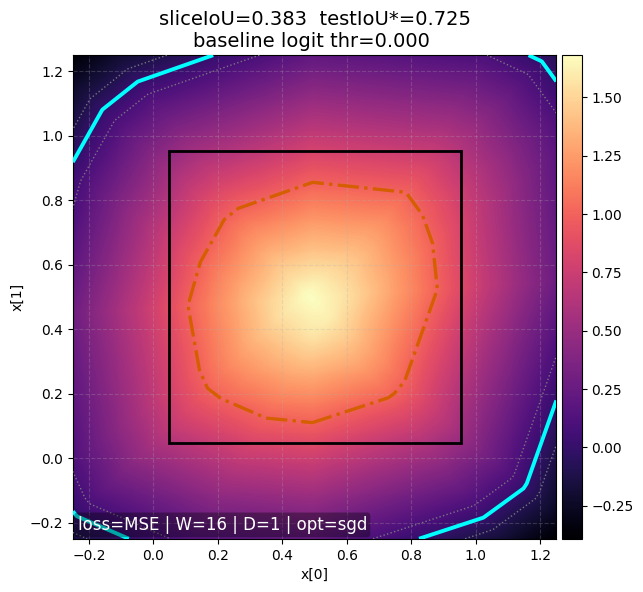} \\
    \scriptsize O &
    \includegraphics[width=0.45\linewidth]{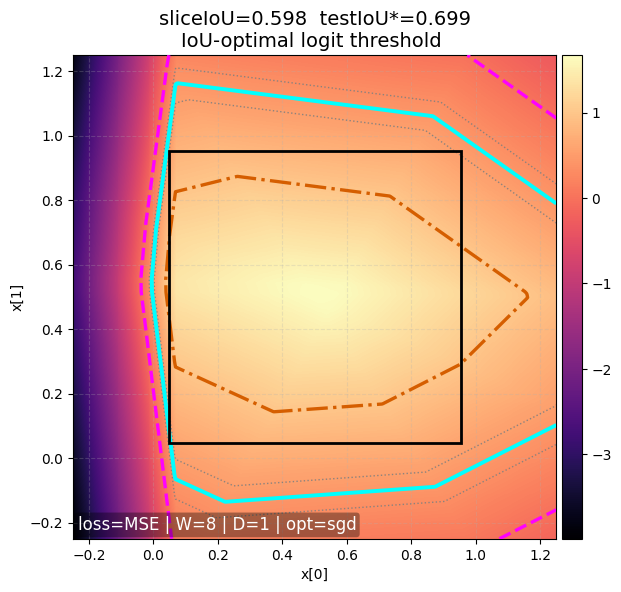} &
    \includegraphics[width=0.45\linewidth]{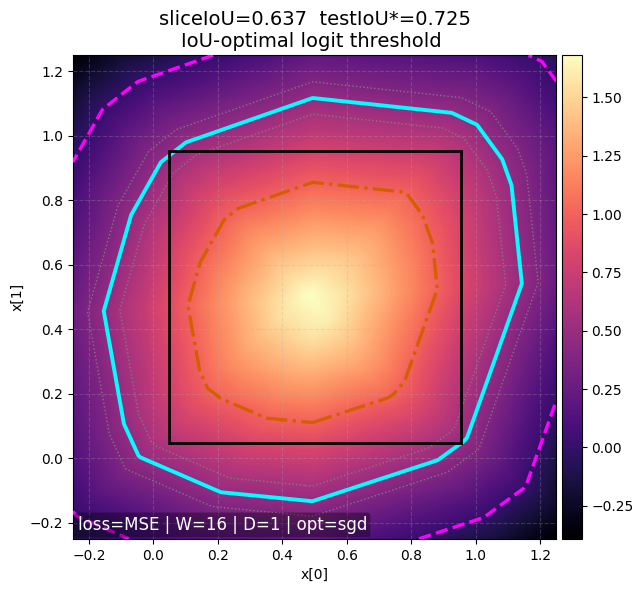} \\
  \end{tabular}
  \caption{\textbf{Raw logits on a 2-D slice across widths (MSE training).}
  Heatmaps show the learned score $f_\theta(x)$ (raw logit output) on the $(x_0,x_1)$ slice
  with $x_2=x_3=x_4=0.5$.
  Row \textbf{M} shows the trained model output.
  Row \textbf{O} overlays the IoU-optimal decision boundary $\{x:f_\theta(x)=\tau^\star\}$,
  where $\tau^\star$ is selected by grid search on a held-out validation split.
  The dashed rectangle is the ground-truth box boundary on the slice.
  Panel headers report $\mathrm{IoU}@\tau{=}0 \rightarrow \mathrm{IoU}^\star$ and $\tau^\star$;
  note that $\mathrm{IoU}@0.5$ corresponds here to the level set at \emph{logit} $0$
  (since no sigmoid/probability is used).
  Across widths, $\mathcal{R}\mathrm{TV}$ increases:
  $24.1$ ($W{=}8$), $39.6$ ($16$), $119.7$ ($32$), $189.2$ ($64$), $264.2$ ($128$), $374.8$ ($256$),
  illustrating a fit--complexity trade-off.}
  \label{fig:nn-box-logits-a}
\end{figure}

\begin{figure}[t]
  \centering
  \setlength{\tabcolsep}{4pt}
  \renewcommand{\arraystretch}{1.0}
  \begin{tabular}{c c c}
    & \scriptsize $W=32$ & \scriptsize $W=64$ \\
    \scriptsize M &
    \includegraphics[width=0.45\linewidth]{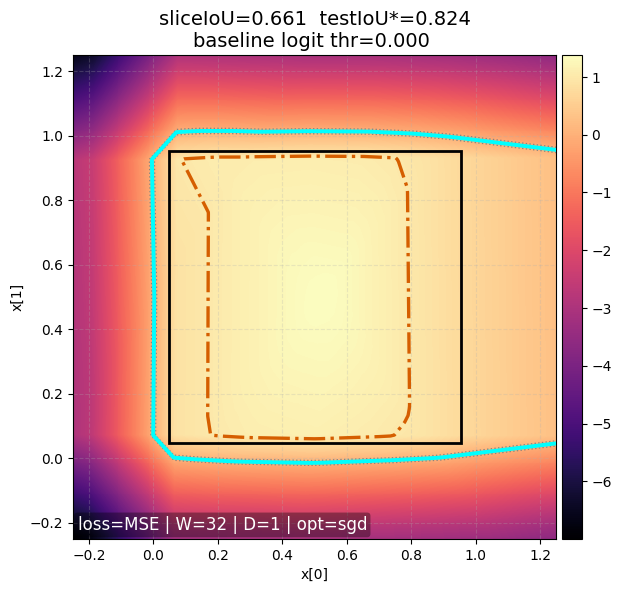} &
    \includegraphics[width=0.45\linewidth]{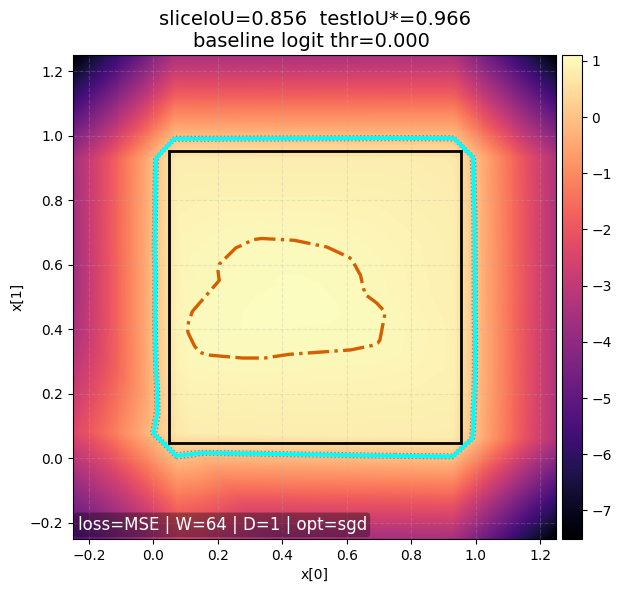} \\
    \scriptsize O &
    \includegraphics[width=0.45\linewidth]{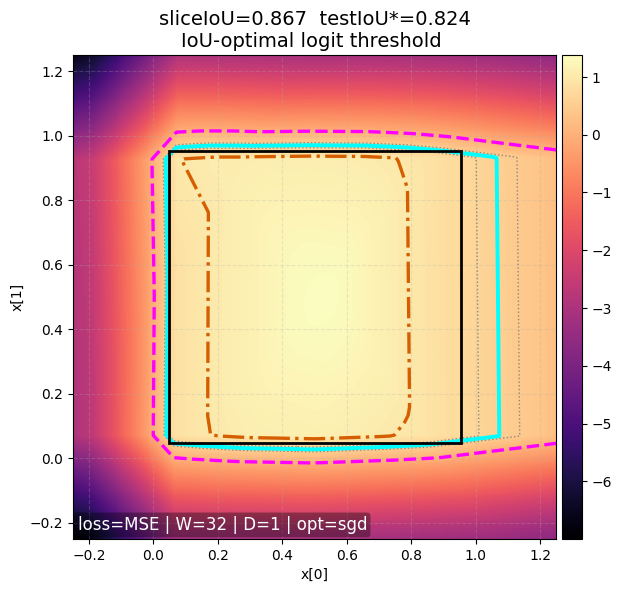} &
    \includegraphics[width=0.45\linewidth]{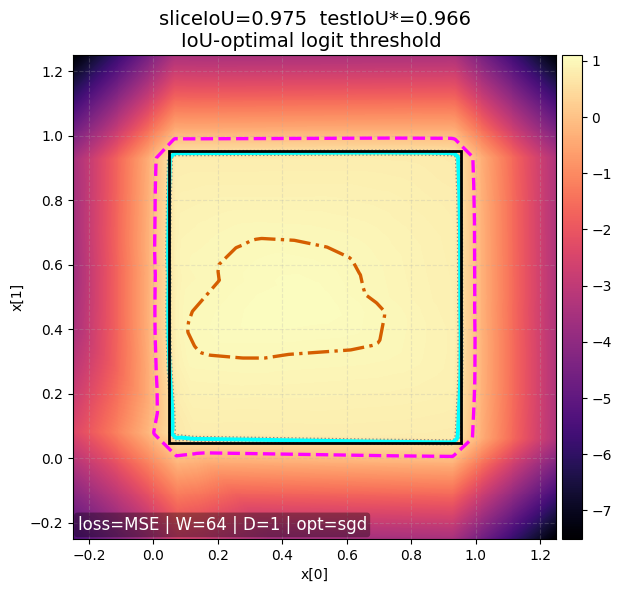} \\
  \end{tabular}
  \caption{\textbf{Raw logits on a 2-D slice across widths (continued).}}
  \label{fig:nn-box-logits-b}
\end{figure}

\begin{figure}[t]
  \centering
  \setlength{\tabcolsep}{4pt}
  \renewcommand{\arraystretch}{1.0}
  \begin{tabular}{c c c}
    & \scriptsize $W=128$ & \scriptsize $W=256$ \\
    \scriptsize M &
    \includegraphics[width=0.45\linewidth]{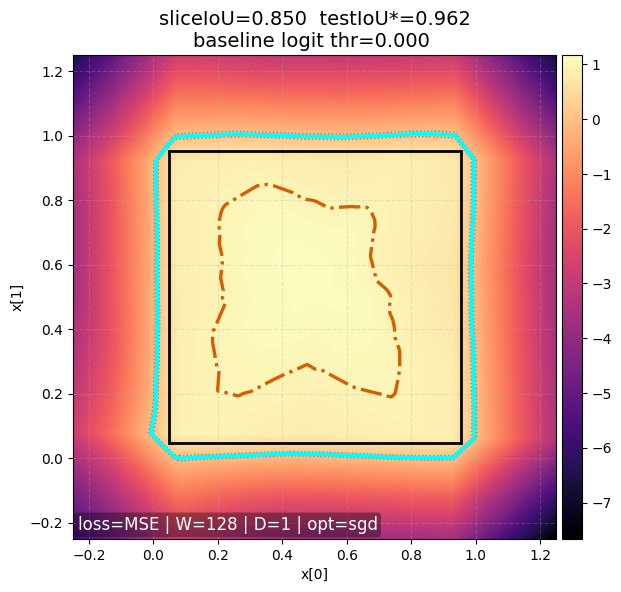} &
    \includegraphics[width=0.45\linewidth]{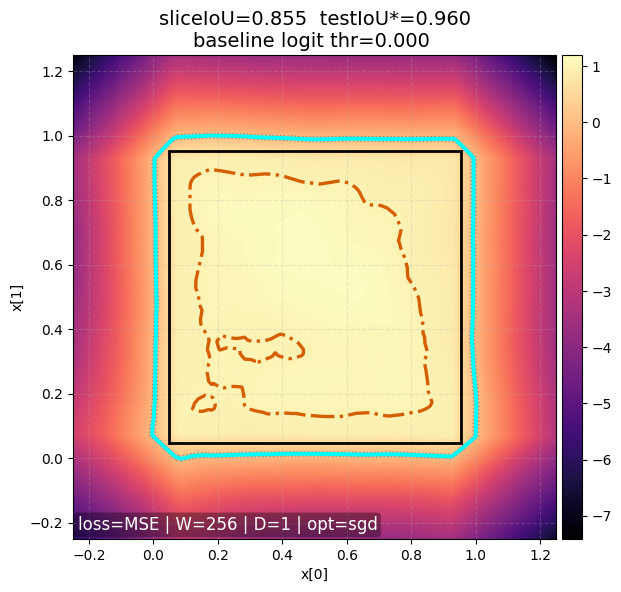} \\
    \scriptsize O &
    \includegraphics[width=0.45\linewidth]{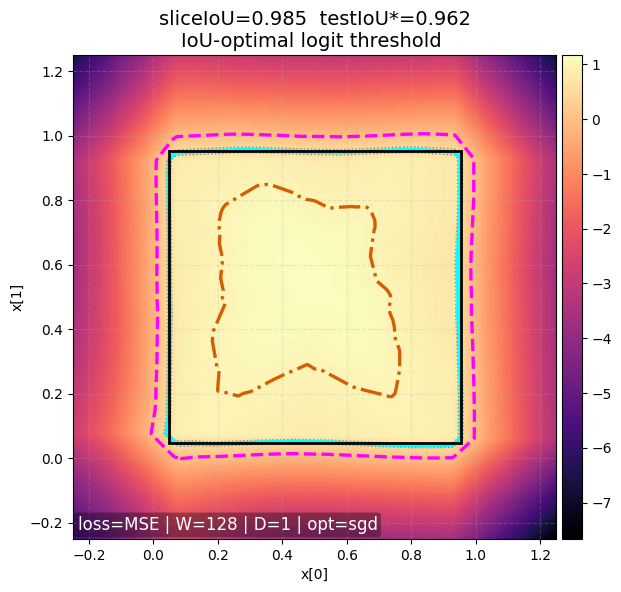} &
    \includegraphics[width=0.45\linewidth]{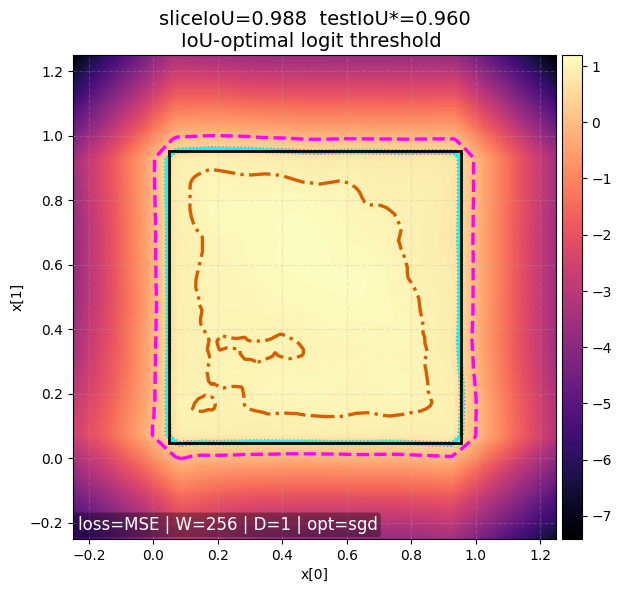} \\
  \end{tabular}
  \caption{\textbf{Raw logits on a 2-D slice across widths (continued).}}
  \label{fig:nn-box-logits-c}
\end{figure}

\begin{figure}[t]
  \centering
  \setlength{\tabcolsep}{4pt}
  \renewcommand{\arraystretch}{1.0}
  \begin{tabular}{c c c}
    & \scriptsize $W=8$ & \scriptsize $W=16$ \\
    \scriptsize M &
    \includegraphics[width=0.45\linewidth]{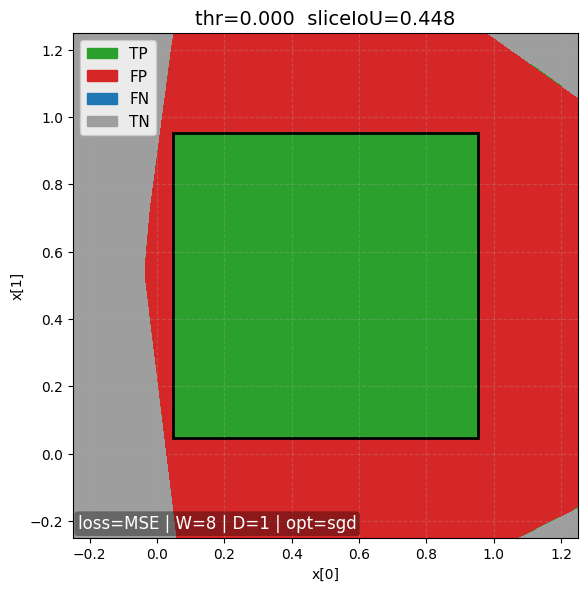} &
    \includegraphics[width=0.45\linewidth]{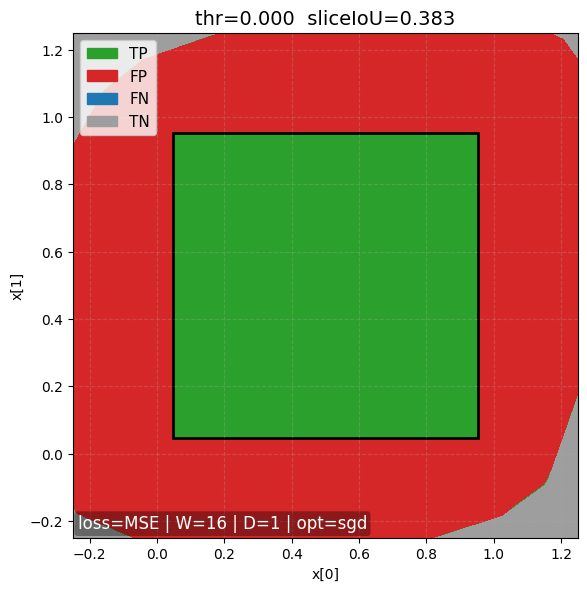} \\
    \scriptsize O &
    \includegraphics[width=0.45\linewidth]{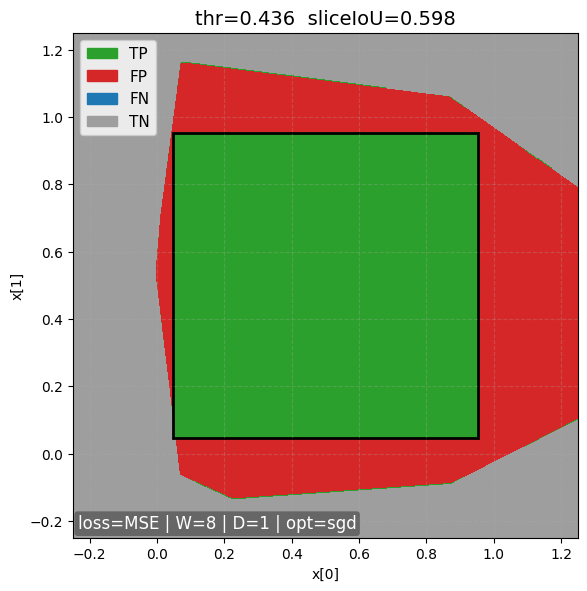} &
    \includegraphics[width=0.45\linewidth]{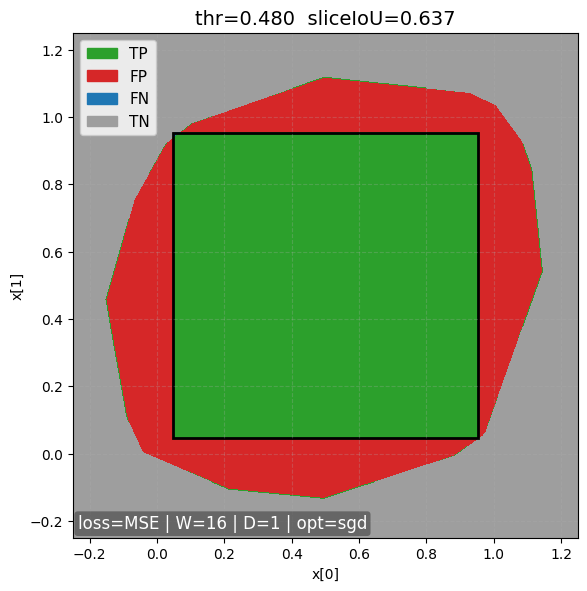} \\
  \end{tabular}
  \caption{\textbf{Thresholded predictions on the same slice.}
  Each subplot shows the predicted-positive region on the $(x_0,x_1)$ slice
  (TP=green, FP=red, FN=blue, TN=gray), with the ground-truth box outlined (dashed).
  Row \textbf{M} uses the fixed threshold $\tau=0$ (the logit threshold corresponding to $p=0.5$).
  Row \textbf{O} uses the IoU-optimal threshold $\tau^\star$ chosen on the validation split.
  Numbers above each panel report the threshold and the resulting test IoU (computed in full $5$-D),
  highlighting that calibration/threshold selection can substantially improve IoU without changing the trained network.}
  \label{fig:nn-box-threshold-a}
\end{figure}

\begin{figure}[t]
  \centering
  \setlength{\tabcolsep}{4pt}
  \renewcommand{\arraystretch}{1.0}
  \begin{tabular}{c c c}
    & \scriptsize $W=32$ & \scriptsize $W=64$ \\
    \scriptsize M &
    \includegraphics[width=0.45\linewidth]{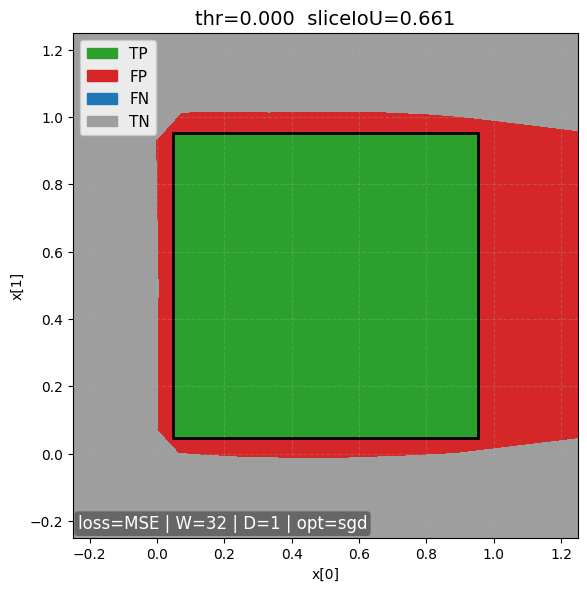} &
    \includegraphics[width=0.45\linewidth]{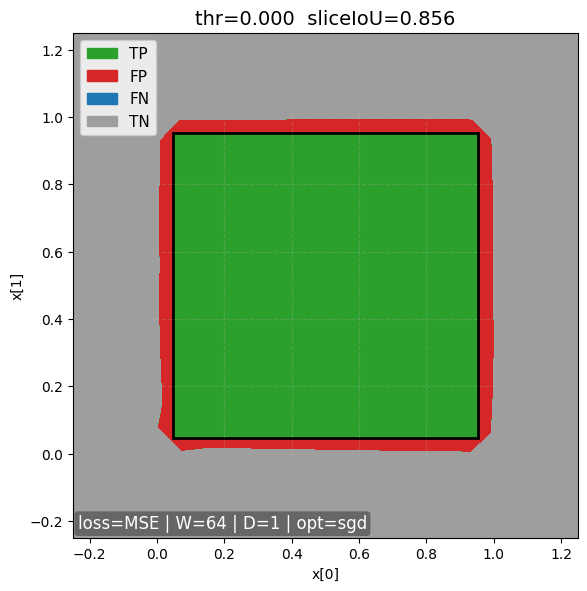} \\
    \scriptsize O &
    \includegraphics[width=0.45\linewidth]{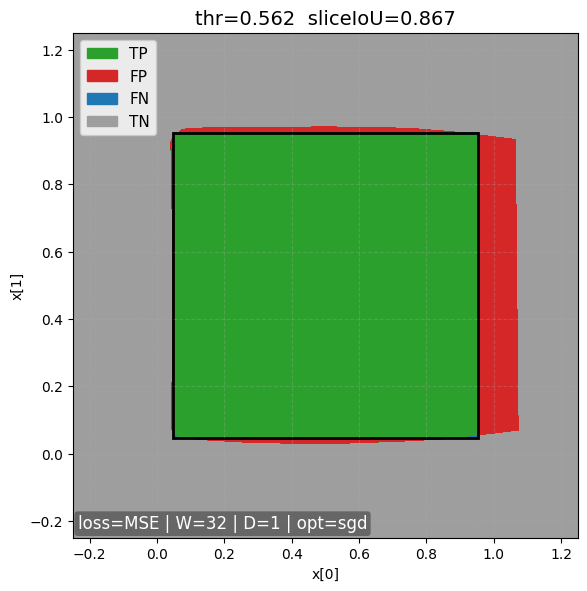} &
    \includegraphics[width=0.45\linewidth]{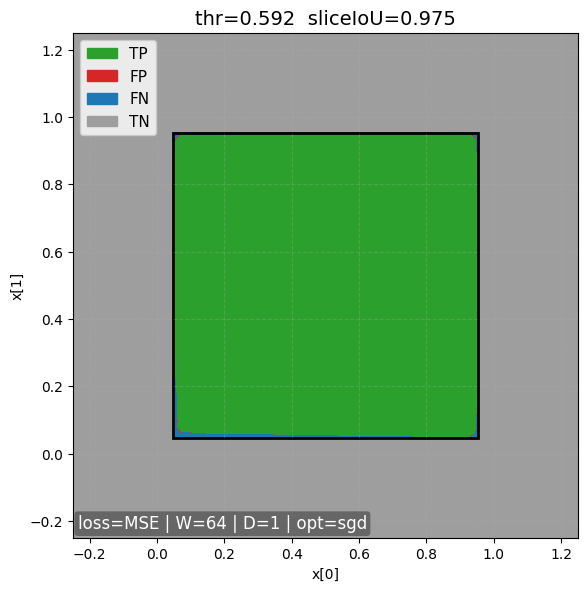} \\
  \end{tabular}
  \caption{\textbf{Thresholded predictions on the same slice (continued).}}
  \label{fig:nn-box-threshold-b}
\end{figure}

\begin{figure}[t]
  \centering
  \setlength{\tabcolsep}{4pt}
  \renewcommand{\arraystretch}{1.0}
  \begin{tabular}{c c c}
    & \scriptsize $W=128$ & \scriptsize $W=256$ \\
    \scriptsize M &
    \includegraphics[width=0.45\linewidth]{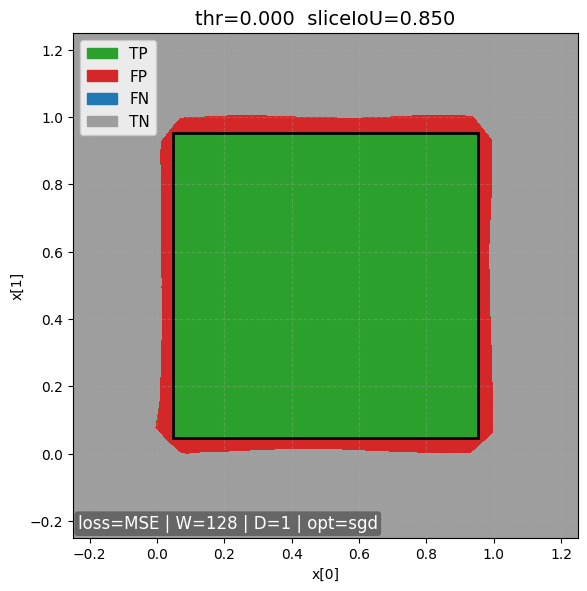} &
    \includegraphics[width=0.45\linewidth]{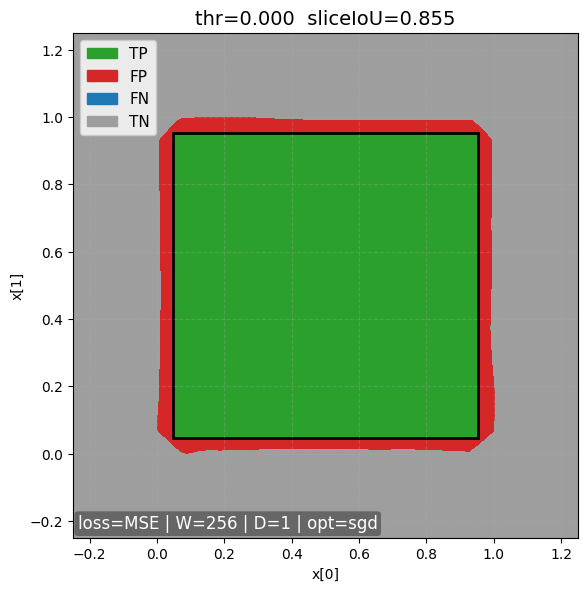} \\
    \scriptsize O &
    \includegraphics[width=0.45\linewidth]{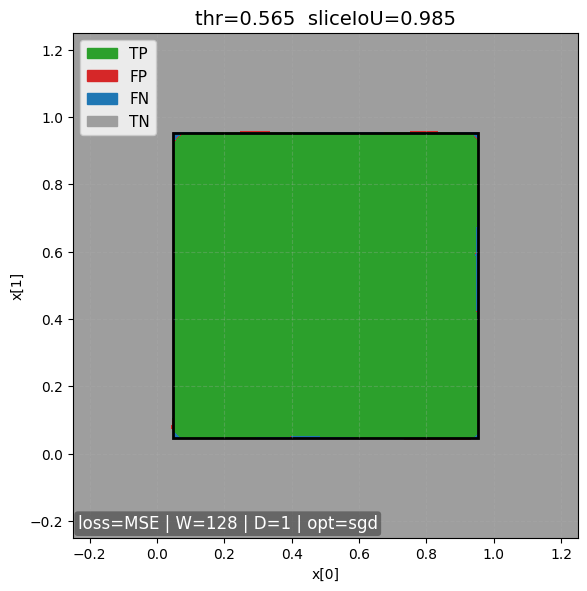} &
    \includegraphics[width=0.45\linewidth]{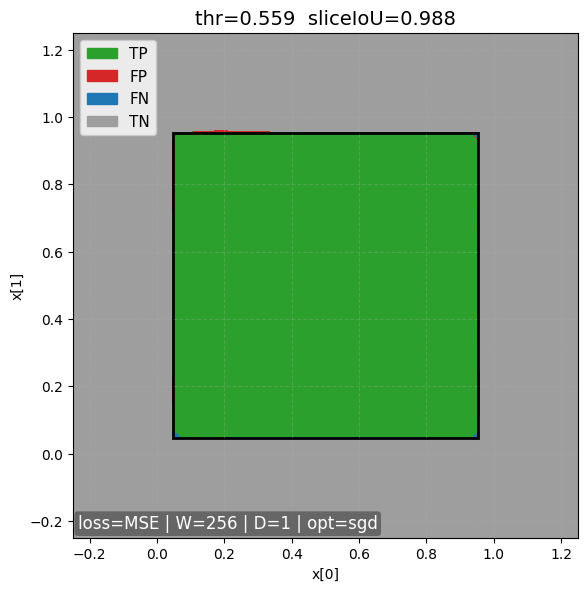} \\
  \end{tabular}
  \caption{\textbf{Thresholded predictions on the same slice (continued).}}
  \label{fig:nn-box-threshold-c}
\end{figure}

\FloatBarrier

%% file: conclusion.tex
\section{Conclusion and Future Work}\label{sec:conclusion}

We examined when accuracy and interpretability can come apart for \emph{shallow} ReLU networks on axis-aligned, decision-tree–type targets, using Radon total variation ($\mathcal{R}\mathrm{TV}$) as a geometric complexity measure. A key point is that two objectives are distinct: matching the \emph{decision set} after thresholding versus matching the \emph{score} itself (calibration / regression). 

On the approximation side, we identified sharp differences between these objectives. The hard indicator $\mathbf 1_A$ lies outside bounded-$\mathcal{R}\mathrm{TV}$ classes in dimension $d>1$, and several common smooth surrogates inherit this behavior, while other smoothings restore finiteness but with an unfavorable dependence on dimension. On the classification side, we exhibited smooth scores that recover $A$ exactly at a fixed threshold and admit quantitative control of $\mathcal{R}\mathrm{TV}$ together with distributional calibration under a tube-mass condition. Our experiments complement these results by showing that near-perfect thresholded accuracy can coexist with meaningful variation in the learned score geometry and complexity.

Overall, this supports a simple takeaway: evaluating only the thresholded classifier can mask substantial differences in the underlying score. When one additionally requires scores that are calibrated and geometrically simple, an explicit accuracy--complexity tension becomes visible.

\paragraph{Future Directions.}
It would be useful to extend these ideas beyond single hidden-layer models and beyond axis-aligned trees, and to study algorithmic mechanisms that bias learning toward low-complexity scores.

\newpage

%% file: toc.tex
\section{A Gap Between Decision Trees and Neural Networks: 
Supplementary Materials}
\vspace*{10mm}
\begin{itemize}
    \item \appref{app:hard}: \textbf{Approximation of hard-threshold decision trees}\vspace{2mm}
    \begin{itemize}
        \item \appref{subapp:step}: Proof of \lemref{lem:step}
        \item \appref{subapp:generalhard}: Proof of \thmref{thm:hard}
    \end{itemize}
    \item \appref{app:sigmoid}: \textbf{Approximation of Sigmoidal smooth decision trees}\vspace{2mm}
    \begin{itemize}
        \item Proof of \thmref{thm:sigmoid}.\vspace{1mm}
    \end{itemize}
    
    \item \appref{app:gauss}: \textbf{Approximation of Gaussian smooth decision trees}\vspace{2mm}
    \begin{itemize}
        \item Proof of \thmref{thm:gaussian}\vspace{1mm}
    \end{itemize}

    \item \appref{app:post-thres}: \textbf{Approximation Post-thresholding}
     \begin{itemize}
        \item Proof of \lemref{lem:single-box-main}\vspace{1mm}
        \item Proof of \thmref{thm:RTV-single-main}\vspace{1mm}
    \end{itemize}
\end{itemize}

%% file: app_DT.tex
\section{Hard Threshold Decision Trees}\label{app:hard}
In this appendix, we provide the proof of the main results as presented in \secref{sec:hard}.

\subsection{$\tv$ of 1-D step functions}\label{subapp:step}

We defined a step function in single dimension as $f_{\sf{step}}: \reals \to \reals$ where
\begin{align*}
    f_{\sf{step}}(x) = \sum_{i=1}^{n} c_i\cdot {1}\curlybracket{x \in (z_i, z_{i+1})}
\end{align*}
for given set of scalars $-\infty < z_1 \le z_2 \le \ldots \le z_N < \infty$

We restate the claim of unboundedness of $\rtv{f_{\sf{step}}}$ with the proof below it.
\begin{lemma}
    $\rtv{f_{\sf{step}}}$ is unbounded.
\end{lemma}

\begin{proof}
    Using Theorem 3.1~\citep{savarese19a}, we note that for a choice of small enough $\epsilon > 0 $
    \begin{align*}
        \rtv{f_{\sf{step}}} &= \int_{- \infty}^\infty |f''_{\sf{step}}(x)| d x\\
        & = \int_{- \infty}^\infty \bigg\lvert\paren{\sum_{i=1}^{n} c_i\cdot \Delta_{i,i+1} \delta_{z_i}(x) }' \bigg\rvert d x\\
        & = \int_{- \infty}^\infty \bigg\lvert\sum_{i=1}^{n} c_i\cdot \Delta_{i,i+1} \delta_{z_i}'(x)  \bigg\rvert d x \\
        & = \sum_{i = 1}^n \int_{z_i - \epsilon}^{z_i + \epsilon} \bigg\lvert c_i \cdot \Delta_{i, i+1} \delta'_{z_i}(x)  \bigg\rvert\, d x\\
        & = \sum_{i = 1}^n c_i \cdot \Delta_{i, i+1} \int_{z_i - \epsilon}^{z_i + \epsilon} \big\lvert  \delta'_{z_i}(x)  \big\rvert\, d x \to \infty
    \end{align*}
    where in the last equation we note that the $\delta'$ is a dipole distribution whose $\ell_1$ norm is unbounded.
\end{proof}

\subsection{$\tv$ of high-dimensional step functions}\label{subapp:generalhard}
In this appendix, we provide the proof of \thmref{thm:hard}.

Throughout we adopt the \emph{unitary} Fourier convention
\begin{align}
    \widehat{g}(\bm{\xi})
       \;:=\;
       (2\pi)^{-d/2}\int_{\mathbb{R}^d} g(x)\,e^{-\mathrm i\,\bm{\xi}^{\top}\x}\,d\x,
       \qquad \bm{\xi}\in\mathbb{R}^d. 
\end{align}
All computations are in the sense of \textbf{tempered distributions}
(Schwartz space dual); every integral we write down exists in that sense.

Thus, for a  decision tree as defined in \secref{sec:setup}:
\begin{align*}
    \hat{f}_{\sf{DT}}(\bm{\xi}) &:= (2\pi)^{-d/2} \int f_{\sf{DT}}(\x) e^{- \mathrm{i} \xi^\top \x} d \x\\
     &= (2\pi)^{-d/2} \int {1}\curlybracket{\x \in A} e^{- \mathrm{i} \bm{\xi}^\top \x} d \x \\
     & = (2\pi)^{-d/2} \int_{A} e^{- \mathrm{i} \xi^\top \x} d \x
\end{align*}

Evaluating $\hat{f}_{\sf{DT}}$ at $\bm{\omega} = \omega\bm{\beta}$, we get
\begin{align}
    \hat{f}_{\sf{DT}}(\omega\bm{\beta}) = (2\pi)^{-d/2} \int_A e^{-\mathrm{i}\omega \bm{\beta}^{\top}\x} d \x \label{eq:fourDT}
\end{align}
Using the Fourier slice theorem~\citep{kak_slaney} we have 
\[
    \mathcal{F}_1\{\mathcal{R}\{f_{\sf{DT}}\}(\bm{\beta}, \cdot)\}(\omega) = \hat{f}_{\sf{DT}}(\omega\bm{\beta}).
\]
This gives the following Radon transform of $f_{\sf{DT}}$: 
\begin{align*}
    \mathcal{R}\{f_{\sf{DT}}\}(\bm{\beta}, t) & = (2\pi)^{-1/2} \int_{\reals} e^{\mathrm{i\omega t}} \hat{f}_{\sf{DT}}(\omega \bm{\beta}) d \omega\\
    & =  (2\pi)^{-(d+1)/2} \int_{\reals} e^{\mathrm{i\omega t}} \int_A e^{-\omega\mathrm{i}\bm{\beta}^{\top}\x} d \x\, d \omega\\
    & = (2\pi)^{-(d+1)/2} \int_A \int_{\reals} e^{\mathrm{i\omega} (t - \bm{\beta}^{\top} \x)}  d \omega\, d \x\\
    & = (2\pi)^{-(d+1)/2} \int_A \delta(t - \bm{\beta}^\top \x) d \x
\end{align*}
Now, we compute the ${d+1}$-derivative of $\mathcal{R}\{f\}(\bm{\beta}, t)$ with respect to $t$ (the integral is defined in the sense of a tempered distribution and follows the convention discussed in \cite{GELFAND19641})
\begin{align*}
    \partial_t \mathcal{R}\{f_{\sf{DT}}\}(\bm{\beta}, t) = (2\pi)^{-(d+1)/2} \int_A \delta'(t - \bm{\beta}^\top\x)\, d \x
\end{align*}
Similarly $(d+1)$th derivative in $t$ is 
\begin{align*}
    \partial_t^{d+1} \mathcal{R}\{f_{\sf{DT}}\}(\bm{\beta}, t) = (2\pi)^{-(d+1)/2} \int_A \delta^{(d+1)}(t - \bm{\beta}^\top\x)\, d \x
\end{align*}
If $d$ is odd, then the second-order Radon domain total variation is the $L^1$-norm of $(d+1)$ derivatives in $t$ of this quantity (see Equation.(28) in \cite{Parhi2020BanachSR}). That is
\begin{align*}
    \rtv{f_{\sf{DT}}} &= c_d (2\pi)^{-(d-1)/2} \int_{\mathbb{S}^{d-1}} \int_{\reals} \bigg\lvert \int_A \delta^{(d+1)}(t - \bm{\beta}^\top\x)\, d \x \bigg \rvert \,d t\, d \bm{\beta}\\
\end{align*}
Lets define for any $\bm{\beta} \in \mathbb{S}^{d-1}$
\begin{align}
    g_{\bm{\beta}}(u) := \int_A \delta(u - \bm{\beta}^{\top} \x)\, d \x \label{eq: defb}
\end{align}
where we can write
\begin{align*}
    g^{(k)}_{\bm{\beta}}(u) = \int_A \delta^{(k)}(u - \bm{\beta}^{\top} \x)\, d \x
\end{align*}
for any $k > 0$.

Now, using the co-area formula \citep{Mattila_1995}, noting that $||\bm{\beta}|| = 1$, we can rewrite \eqnref{eq: defb} as 
\begin{align*}
    \int_A \delta(u - \bm{\beta}^{\top} \x)\, d \x = \int_{\bm{\beta}^{\top} \x = u} {1}\{\x \in A\}\, d \sigma(\x)
\end{align*}
where $d\sigma$ denotes the $(d-1)$-dimensional Hausdorff measure on the hyperplane \(\{\x:\bm{\beta}^{\top}\x=u\}\)

Now, we will show there exists, $\bm{\beta_0} \in \mathbb{S}^{d-1}$, and scalar $\epsilon > 0$ such that for 
\begin{align}
   \rtv{f_{\sf{DT}}} \ge  \int_{\cB(\bm{\beta_0}, \epsilon)} \int_{\reals} |g^{(d+1)}_{\bm{\beta}}(t)|\, d t\, d \bm{\beta} \to \infty \label{eq:DTLB}
\end{align}
where $\cB_2(\bm{\beta_0}, \epsilon) := \curlybracket{\bm{\beta} : ||\bm{\beta} - \bm{\beta}_0|| \le \epsilon}$.

Without loss of generality, assume that $A$ is axes-aligned to eigendirections of $\reals^d$--- $\be_1, \be_2,\ldots, \be_d$. Thus, consider the case where $\bm{\beta_0} := \be_1$. 

Now,
\begin{align*}
    g_{\bm{\beta}_0 + \Delta}(u) = \begin{cases}
        0, &\textnormal{ if } u \in (a_1, b_1)^{c}\\
        Vol\paren{\curlybracket{{(\bm{\beta}_0 + \Delta})^\top \x = u} \cap A} &\textnormal{ if } u \in (a_1, b_1)\\
    \end{cases}
\end{align*}
Note, if $\Delta = \bm{0}$, then $g_{\bm{\beta}_0}$ has a sharp discontinuity at $u = a_1$. But if $\Delta \neq \bm{0}$ and $||\bm{\beta}_0 - \Delta|| \le \epsilon$, $g_{\bm{\beta}_0 + \Delta}$ varies over the real line smoothly. But we can control the jump around $(a_1 - \epsilon, a_1 + \epsilon)$. Note that,
\begin{align*}
    \lim_{\epsilon \to 0} g_{\bm{\beta}_0 + \Delta} = g_{\bm{\beta}_0}
\end{align*}

Fix a Gaussian mollifier $\eta \in C_c^\infty(\mathbb R)$, $\eta(u) := \pi^{-1/2}e^{-u^2}$ and put
$\eta_\epsilon(u)=\epsilon^{-1}\eta(u/\epsilon)$,
$g_{\bm\beta,\epsilon}=g_{\bm\beta}*\eta_\epsilon$.
If $h(u)=H\mathbf1_{\{u\ge0\}}$ then
$\max_u|(h*\eta_\epsilon)^{(k)}(u)|=H\,\epsilon^{-(k+1)}\max_s|\eta^{(k)}(s)|$.

Every slice with $\bm\beta\in\mathcal B(\bm\beta_0,\epsilon)$
still contains a jump of height at least $H/2$, where
\(H=\prod_{j=2}^{d}(b_{j}-a_{j})\).
With $k=d+1$ this gives
\[
\max_t|g^{(d+1)}_{\bm\beta,\epsilon}(t)|\ge C\,\epsilon^{-(d+2)},\qquad
\int_{\mathbb R}|g^{(d+1)}_{\bm\beta,\epsilon}(t)|\,dt\ge C\,\epsilon^{-(d+1)}.
\]

Hence
\[
\int_{\mathcal B(\bm\beta_0,\epsilon)}\!\!\int_{\mathbb R}
     |g_{\bm\beta, \epsilon}^{(d+1)}(t)|\,dt\,d\bm\beta
   \;\ge\;
   C'\,\epsilon^{\,d-1}\,\epsilon^{-(d+1)}
   =\frac{C'}{\epsilon^{2}}\xrightarrow{\epsilon\to0}\infty.
\]

But note that as $\epsilon$ tends to 0, the mollified function $g_{\beta,\epsilon}^{(d+1)}$ tends to $g_{\beta}^{(d+1)}$ in distribution. But then this implies that over the convex slope $\cB_2(\bm{\beta}_0,\epsilon)$ 
\begin{align*}
    \int_{\mathcal B(\bm\beta_0)}\!\!\int_{\mathbb R}
     |g_{\bm\beta, \epsilon}^{(d+1)}(t)|\,dt\,d\bm\beta \rightarrow \infty 
\end{align*}
as $\epsilon \to 0$. Hence, using the bound in \eqnref{eq:DTLB}, $\rtv{f_{\sf{DT}}}$ is unbounded.

\newpage

%% file: app_Sigmoid.tex
\section{Sigmoidal Smoothing}\label{app:sigmoid}
In this appendix, we will provide the proof of \thmref{thm:sigmoid} as stated in \secref{sec:smooth} on the approximation of Sigmoid smooth decision trees. 

\subsection{Approximating the $\tv$ for sigmoidal smooth decision trees}

For the sake of clarity, in \secref{sec:smooth}, we analyse the function
\begin{align}
    f_{\sf{DT},\gamma}(\x)
     \;=\;
     \prod_{i=1}^{D}\sigma_\gamma\!\bigl(\w_i^{\top}\x+b_i\bigr),
      \qquad
      \x\in\mathbb{R}^d,\ \gamma>0, \label{eq:sig-app}
\end{align}
where
\[
    \sigma_\gamma(z)\;=\;\frac{1}{1+e^{-z/\gamma}}
\]
is a \emph{scaled logistic},
and $\w_i\in\mathbb{R}^d$, $b_i\in\mathbb{R}$ are the split normals and thresholds
at depth $i$.

Now, we show the proof of the result on the Fourier transform of Sigmoidal smooth decision trees. First, we restate the result on the Fourier transform of a shifted, scaled sigmoid, then provide the proof below.

\begin{lemma}\label{lem:1d-pair}
For any $\gamma>0$ and $b\in\mathbb{R}$,
\begin{align*}
    \widehat{\sigma_\gamma(\,\cdot+b)}(\omega)
       \;=\;
       e^{-\mathrm i\,\omega b}\Bigl[
           \frac{\pi}{2}\,\delta(\omega)
           +\frac{\mathrm i\gamma\pi}{\sinh(\pi\gamma\omega)}
       \Bigr],
    \qquad \omega\in\mathbb{R}.
\end{align*}
\end{lemma}

\begin{proof}

First, note that, by translation invariance,
\[
   \widehat{\sigma_\gamma(\,\cdot+b)}(\omega)
      =e^{-\mathrm i\omega b}\,\widehat{\sigma_\gamma}(\omega),
\]
so it suffices to compute $\widehat{\sigma_\gamma}$.

We can write the sigmoid as 
\[
   \sigma_\gamma(z)=\tfrac12+\tfrac12\tanh\!\Bigl(\tfrac{z}{2\gamma}\Bigr).
\]

For the constant part
\(\widehat{\tfrac12}=\dfrac{\pi}{2}\delta(\omega)\).

For the hyperbolic–tangent part use the table entry  
\[
   \int_{-\infty}^{\infty}\tanh u\,e^{-\mathrm i\Omega u}\,du
        =\frac{\pi\mathrm i}{\sinh(\pi\Omega/2)}
   \qquad(\text{Bateman Vol 1, \S4.9, (9)~\citep{bateman1954tables}}),
\]
together with the scaling rule
\(
   \widehat{f(ax)}(\omega)=\dfrac1{|a|}\widehat{f}\!\bigl(\tfrac{\omega}{a}\bigr)
\) for $a \neq 0$. 
Choosing \(a=2\gamma\) gives
\[
   \widehat{\tanh(\,\cdot/2\gamma)}(\omega)
        =2\gamma\,\frac{\pi\mathrm i}{\sinh(\pi\gamma\omega)}.
\]

Combining the two terms and re-inserting the phase factor gives- for all $\omega \in \reals$
\[
   \widehat{\sigma_\gamma(\,\cdot+b)}(\omega)
     =e^{-\mathrm i\omega b}\Bigl[
         \frac{\pi}{2}\,\delta(\omega)
         +\frac{\mathrm i\gamma\pi}{\sinh(\pi\gamma\omega)}
       \Bigr],
\]
as claimed.
\end{proof}

In the following subsection, we consider the geometry of the splits direction $\{\w_i\}_{i=1}^{D}\subset\mathbb R^{d}$. 

\subsection{Geometry of the split directions}

First, consider rewriting all the split normals into a matrix
\[
      W\;:=\;
      \begin{bmatrix}
         \w_1^{\top}\\[-2pt] \vdots\\[-2pt] \w_D^{\top}
      \end{bmatrix}
      \in\mathbb{R}^{D\times d},
\]
Denote by
\[
      r\;:=\;\operatorname{rank}W\qquad
      (0\le r\le \min\{D,d\})
\]
the dimension of their span
\[
      \mathcal S\;:=\;\operatorname{span}\{\w_1,\dots,\w_D\}
                 \subset\mathbb{R}^d.
\]

\subsection{Rotating into $\mathcal S\oplus\mathcal S^{\perp}$}
Now, we will rewrite the product in the expression of Sigmoidal smooth decision tree in \eqnref{eq:sig} corresponding to the active directions in the span of the split normals $\cS$.

First, note that $\mathcal S$ is $r$–dimensional, so we can pick an orthogonal matrix
\(
      \mR=[\mR_{\parallel}\;\mR_{\perp}]
\)
such that
\[
      \operatorname{im}\mR_{\parallel}=\mathcal S,
      \quad
      \operatorname{im}\mR_{\perp}=\mathcal S^{\perp}.
\]
\textbf{Notation.}
We write every point $\x\in\mathbb{R}^d$ and every frequency $\bm{\xi}$
in these coordinates:

\[
      \x
      =\mR\!\begin{bmatrix} \x_{\parallel}\\ \x_{\perp}\end{bmatrix},
      \quad
      \x_{\parallel}\in\mathbb{R}^{r},\ \x_{\perp}\in\mathbb{R}^{d-r};
      \hspace{2em}
      \bm{\xi}
      =\mR\!\begin{bmatrix} \bm{\eta}\\ \bm{\zeta}\end{bmatrix},
      \quad
      \bm{\eta}\in\mathbb{R}^{r},\ \bm{\zeta}\in\mathbb{R}^{d-r}.
\]

Now, because each $\w_i$ lies \emph{inside} $\mathcal S$ we have
$\w_i^{\top}\mR_{\perp}=0$; hence
\begin{align}
    \w_i^{\top}\x+b_i
      =\w_i^{\top}\mR\!\begin{bmatrix}\x_{\parallel}\\ \x_{\perp}\end{bmatrix}+b_i =\underbrace{(\w_i^{\top}\mR_{\parallel})}_{=:\bm{a}_i^{\top}}\!\x_{\parallel}+b_i,
      \quad \bm{a}_i\in\mathbb{R}^{r}. \label{eq:decompmat}
\end{align}
Thus $f_{\sf{DT},\gamma}$ depends \emph{only} on the $\x_{\parallel}$–coordinates:
\[
     f_{\sf{DT}, \gamma}\bigl(\mR[\x_{\parallel},\x_{\perp}]^{\top}\bigr)
     =\prod_{i=1}^{D}\sigma_\gamma(\bm{a}_i^{\top}\x_{\parallel}+b_i).
\]

\subsection{Splitting the Fourier integral}

By definitions, we know that the Fourier transform of $f_{\sf{DT},\gamma}$ is given by
\begin{align}
    \widehat{f_{\sf{DT}, \gamma}}(\xi) = (2\pi)^{-d/2} \int f_{\sf{DT},\gamma}(\x) e^{- \mathrm{i} \bm{\xi}^\top \x}\, d \x = (2\pi)^{-d/2} \int \Bigl[\prod_{i=1}^{D}\sigma_\gamma(\bm{a}_i^{\top}\x_{\parallel}+b_i)\Bigr] \cdot e^{- \mathrm{i} \bm{\xi}^\top \x}\, d \x \label{eq:foursig-int}
\end{align}
Now, insert the rotated coordinates into the definition
\begin{align}
\widehat{f_{\sf{DT}, \gamma}}(\xi)
   &=(2\pi)^{-d/2}
     \int_{\mathbb{R}^{d}}
       \Bigl[\prod_{i=1}^{D}\sigma_\gamma(a_i^{\top}t+b_i)\Bigr]\,
       e^{-\mathrm i\,(\eta^{\top}t+\zeta^{\top}z)}
       \,dt\,dz                     \nonumber\\
   &=(2\pi)^{-d/2}
     \Bigl[%
       \underbrace{\int_{\mathbb{R}^{d-r}}
                    e^{-\mathrm i\,\zeta^{\top}z}\,dz}
                  _{=\;(2\pi)^{(d-r)/2}\delta(\zeta)}
     \Bigr]
     \int_{\mathbb{R}^{r}}G(t)\,e^{-\mathrm i\,\eta^{\top}t}\,dt,
     \label{eq:split}
\end{align}
where we define
\(
     G(t):=\prod_{i=1}^{D}\sigma_\gamma(a_i^{\top}t+b_i).
\)

\medskip
\noindent
The outer integral in \eqnref{eq:split} produced a \emph{Dirac delta}
\(\delta(\zeta)=\delta(P_{\mathcal S^{\perp}}\xi)\)
that kills every frequency component \emph{outside} $\mathcal S$; we are left
with an $r$‑dimensional Fourier transform of $G$ inside $\mathcal S$.

Then, \eqnref{eq:split} becomes
\begin{equation}
   \widehat{f_{\sf{DT},\gamma}}(\xi)
      =(2\pi)^{-r/2}\,
       \delta\!\bigl(P_{\mathcal S^{\perp}}\xi\bigr)\,
       \underbrace{\int_{\mathbb{R}^{r}}G(t)\,e^{-\mathrm i\,\eta^{\top}t}\,dt}
                  _{=:H(\eta)}.
   \label{eq:reduced-integral}
\end{equation}

\subsection{Expressing $H(\eta)$ as a $D$-fold convolution}
\label{sec:H-as-convolution}
Now, consider the linear projection of $t$ wrt the matrix $A$: \(t\mapsto s:=A t\), where $A$ is formed as the \emph{row-stacked} matrix
\[
      A := \begin{bmatrix} a_1^{\top}\\[-2pt]\vdots\\[-2pt] a_D^{\top}\end{bmatrix}
          \;\in\;\mathbb R^{D\times r},
      \qquad a_i^{\top}=\w_i^{\top}R_{\parallel}\;(1\times r).
\]
Because \(\operatorname{rank}A=r\) (same as \(W\)), the linear map
\(t\mapsto s:=At\) is injective, sending \(t\in\mathbb R^{r}\) to a
vector \(s=(s_1,\dots,s_D)^{\top}\in\mathbb R^{D}\) whose $i$-th entry is
\(s_i=a_i^{\top}t\).
In these $s$–coordinates the kernel factorises:
\[
    G(t)=\prod_{i=1}^{D}\sigma_\gamma\bigl(s_i+b_i\bigr).
\]

Solve for \(t\) by left-multiplying with the Moore–Penrose inverse
\(A^{+}=(A^{\top}A)^{-1}A^{\top}\in\mathbb R^{r\times D}\):
\[
     t=A^{+}s, \qquad
     dt=\bigl|\det(A^{\top}A)\bigr|^{-1/2}\,ds .
\]
Note that \(A^{\top}A\) is \(r\times r\), so the determinant is well defined. With \(t=A^{+}s\) we have
\[
      \eta^{\top}t
        =\eta^{\top}(A^{\top}A)^{-1}A^{\top}s
        =\underbrace{\bigl(A(A^{\top}A)^{-1}\eta\bigr)^{\top}}_{=:u^{\top}}
          s,
      \quad
      u:=A(A^{\top}A)^{-1}\eta\;\in\mathbb R^{D}.
\]

Now, we show one-dimensional convolution in each coordinate. Insert these expressions into \(H(\eta)\):
\begin{align}
  H(\eta)
    &=\bigl|\det(A^{\top}A)\bigr|^{-1/2}
      \int_{\mathbb R^{D}}
         \Bigl[\prod_{i=1}^{D}\sigma_\gamma(s_i+b_i)\Bigr]\,
         e^{-\mathrm i\,u^{\top}s}\,ds                                  \\[2pt]
    &\stackrel{\text{\tiny F.T.}}{=}
      \bigl|\det(A^{\top}A)\bigr|^{-1/2}
      \Bigl(K_{\gamma,1}*K_{\gamma,2}*\cdots*K_{\gamma,D}\Bigr)(u),
      \label{eq:H-conv-correct}
\end{align}
where \(K_{\gamma,i}=\widehat{\sigma_\gamma(\,\cdot+b_i)}\)
(from Lemma \ref{lem:1d-pair}).  Each convolution is in the
\(\mathbb R\)–variable corresponding to the $i$-th coordinate.



The change of variables \(t\mapsto s\) contributed the Jacobian
\(|\det(A^{\top}A)|^{-1/2}\).
Because \(A=W R_{\parallel}\) and \(R_{\parallel}\) is orthogonal on
\(\mathcal S\), one has
\(
   A^{\top}A = R_{\parallel}^{\top}(W^{\top}W)R_{\parallel}
            = W^{\top}W|_{\mathcal S},
\)
so \(|\det(A^{\top}A)| = |\det(WW^{\top})|\).
Define the \emph{left} pseudo-inverse
\[
     W^{+} := (WW^{\top})^{-1}W \;\in\; \mathbb R^{D\times d}.
\]
Combining \eqref{eq:reduced-integral} with
\eqref{eq:H-conv-correct} now yields the following arbitrary split formula.

\begin{theorem}[Fourier transform, arbitrary split directions]
\label{thm:general-correct}
Let\/ $\w_1,\dots,\w_D\in\mathbb R^{d}$, $b_1,\dots,b_D\in\mathbb R$ and
$\gamma>0$.  Put  
\(
     W=\bigl[\w_1\,\cdots\,\w_D\bigr]^{\!\top}\in\mathbb R^{D\times d},
\;
     r=\operatorname{rank}W,
\;
     \mathcal S=\operatorname{span}\{\w_i\}.
\)
With the unitary convention  
\[
   \widehat g(\xi)=(2\pi)^{-d/2}\!\int_{\mathbb R^{d}}
         g(x)\,e^{-\mathrm i\xi^{\top}x}\,dx,
\]
the the Sigmoid smooth decision tree with depth $D > 0$
\[f_{\sf{DT},\gamma}(x)=\prod_{i=1}^{D}\sigma_\gamma(\w_i^{\top}x+b_i)\]
has Fourier transform
\[
   \widehat{f_{\sf{DT},\gamma}}(\xi)
        =(2\pi)^{-r/2}\,
         |\,\det(WW^{\top})|^{-1/2}\,
         \delta\!\bigl(P_{\mathcal S^{\perp}}\xi\bigr)\;
         \Bigl(K_{\gamma,1}*\,\cdots\,*K_{\gamma,D}\Bigr)\!\bigl(W^{+}\xi\bigr),
\]
where
\(
   K_{\gamma,i}(u)=e^{-\mathrm i b_i u}
                   \bigl[\frac{\pi}{2}\delta(u)
                         +\frac{\mathrm i\gamma\pi}{\sinh(\pi\gamma u)}\bigr]
\)
is the 1-D kernel from Lemma \ref{lem:1d-pair}.
\end{theorem}
\begin{remark}
In the theorem above, we assumed that the split normals $\curly{\w_i}$ are distinct. In the case when the directions are not distinct, one can consider a maximally linearly independent set of split normals to obtain similar results as above. Thus, the convolution of the 1-D kernels involves multiplicity of the split normals in the depth product.
\end{remark}



Suppose now that the $\w_i$ are \emph{orthonormal and distinct}.
Then $r=D\le d$ and
\(
   W^{+}=W^{\top},\;
   WW^{\top}=I_D,\;
   A=I_D
\).
Because each $K_{\gamma,i}$ acts on an \emph{independent coordinate}
(the $i$‑th standard basis vector), convolutions reduce to ordinary
point‑wise products:
\[
     K_{\gamma,1}*\cdots*K_{\gamma,D}
       =K_{\gamma,1}\,\cdots\,K_{\gamma,D}.
\]
This yields the following product‑form for orthonormal splits in the Fourier transform.

\begin{corollary}[Orthonormal $\{\w_i\}$, $D\le d$]
\begin{align}
    \widehat{f_{\sf{DT},\gamma}}(\xi)
       =(2\pi)^{-D/2}\,
        \delta\!\bigl(P_{\mathcal S^{\perp}}\xi\bigr)\,
        \prod_{i=1}^{D}
           e^{-\mathrm i\,b_i\eta_i}
           \Bigl[\frac{\pi}{2}\,\delta(\eta_i)
                 +\frac{\mathrm i\,\gamma\pi}{\sinh(\pi\gamma\eta_i)}\Bigr],
        \quad \eta_i=\w_i^{\top}\xi. \label{eq:foursig-closed}
\end{align}
\end{corollary}

With this we provide the proof of the main theorem.

\subsection{Proof of \thmref{thm:sigmoid}}
Here, we prove \thmref{thm:sigmoid} using the Fourier representation of the Sigmoidal smooth decision tree in \eqnref{eq:foursig-closed}. We assume that $D = d$ for simplicity; the case $D < d$ follows similarly by considering the subspace spanned by the split directions as in the previous subsection.
\paragraph{Notation.}
Denote by
\[
  \lambda_i(\beta)=\w_i^{\!\top}\beta
  \quad(i=1,\dots,D),
\]
and write \(d\sigma(\beta)\) for the surface measure on the
unit sphere \(\mathbb S^{d-1}\subset\mathbb R^{d}\).

Note that, the tempered distribution \(f_{\sf{DT},\gamma}\) has Fourier transform
\[
  \widehat{f_{\sf{DT}, \gamma}}(\xi)      
       =(2\pi)^{-\frac{D}{2}}\,
        \delta\!\bigl(P_{\mathcal S^{\perp}}\xi\bigr)\,
        \prod_{i=1}^{D}
           e^{-\mathrm i\,b_i \w_i^{\!\top}\xi}
           \Bigl[\frac{\pi}{2}\,\delta\!\bigl(\w_i^{\!\top}\xi\bigr)
                 +\frac{\mathrm i\,\gamma\pi}
                        {\sinh\bigl(\pi\gamma \w_i^{\!\top}\xi\bigr)}\Bigr],
\]
so \(\widehat{f_{\sf{DT}, \gamma}}\) is supported on \(\mathcal S\).

Now, evaluating the Fourier transform on the $\omega\bm{\beta}$ we get
\begin{align*}
    \widehat{f_{\sf{DT}, \gamma}}(\omega\bm{\beta})      
       =(2\pi)^{-\frac{D}{2}}\,
        \delta\!\bigl(P_{\mathcal S^{\perp}}\omega\bm{\beta}\bigr)\,
        \prod_{i=1}^{D}
           e^{-\mathrm i\,b_i \omega\w_i^{\!\top}\bm{\beta}}
           \Bigl[\frac{\pi}{2}\,\delta\!\bigl(\omega\w_i^{\!\top}\bm{\beta}\bigr)
                 +\frac{\mathrm i\,\gamma\pi}
                        {\sinh\bigl(\pi\gamma \omega\w_i^{\!\top}\bm{\beta}\bigr)}\Bigr]
\end{align*}

Now, we can write the 1D inverse Fourier transform to solve for the Radon transform of ${f_{\sf{DT}, \gamma}}$
\begin{align*}
    \cR\curlybracket{f_{\sf{DT}, \gamma}}(\bm{\beta},t) = (2\pi)^{-1/2}\int_{\reals} e^{\mathrm i \omega t} \widehat{f_{\sf{DT},\gamma}}(\omega \bm{\beta})\, d \omega
\end{align*}
Because of the factor $\delta\!\bigl(P_{\mathcal S^{\perp}}\xi\bigr)$ in \eqnref{eq:foursig-closed},
the integrand is non–zero \emph{only if}
$\beta_{\perp}=0$, i.e.\ $\beta\in\mathcal S$. Hence
\[
   \cR\curlybracket{f_{\sf{DT}, \gamma}}(\bm{\beta},t)=0
   \quad\text{unless}\quad\beta\in\mathcal S\cap\mathbb S^{d-1}.
\]

Because \(\widehat{f_{\sf{DT}, \gamma}}(\omega\bm{\beta})\) has at most polynomial
growth in \(\omega\), we may differentiate under the integral:
\begin{align}
    \partial_{t}^{\,d+1}\mathcal R\curly{f_{\sf{DT}, \gamma}}(\bm{\beta},t)
  =(2\pi)^{-1/2}\!\int_{\mathbb R}
        (\mathrm i\omega)^{d+1}\,
        e^{\mathrm i\omega t}\,
        \widehat{f_{\sf{DT}, \gamma}}(\omega\bm{\beta})\,d\omega.\label{eq:sigRadon}
\end{align}
Now, integrating absolute value of LHS in \eqnref{eq:sigRadon},
\begin{align}
    ||\partial_{t}^{\,d+1}\mathcal R\curly{f_{\sf{DT}, \gamma}}(\bm{\beta},\cdot)||_{L^1_t} & = (2\pi)^{-1/2} \int_{\reals} \Bigl| \int_{\mathbb R}
        (\mathrm i\omega)^{d+1}\,
        e^{\mathrm i\omega t}\,
        \widehat{f_{\sf{DT}, \gamma}}(\omega\bm{\beta})\,d\omega \Bigr|\, d t \nonumber\\
        & = (2\pi)^{-1/2} \int_{\reals} \Bigl| \int_{\mathbb R}
        (\omega)^{d+1}\,
        e^{\mathrm i\omega t}\,
        \widehat{f_{\sf{DT}, \gamma}}(\omega\bm{\beta})\,d\omega \Bigr|\, d t \label{eq:expand1}
\end{align}
Now, note that expanding the full form of $\widehat{f_{\sf{DT},\gamma}}$ in \eqnref{eq:expand1}, we can eliminate the
\(\tfrac{\pi}{2}\delta(\omega\w_i^{\top} \bm{\beta})\) terms due to the application of the Sifting property of Dirac delta on \(|\omega|^{d+1}\delta(\omega)\). Hence, we can simplify the \eqnref{eq:expand1} as follows:
\begin{align}
\centering
    \hspace{-10mm}||\partial_{t}^{\,d+1}\mathcal R\curly{f_{\sf{DT}, \gamma}}(\bm{\beta},\cdot)||_{L^1_t} = (2\pi)^{-(D+1)/2} (\gamma\pi)^D\int_{\reals} \Bigl| \int_{\mathbb R}
        (\omega)^{d+1}\,
        e^{\mathrm i\omega \paren{t - \sum_{i=1}^D b_i \w_i^{\top} \bm{\beta}}}\,
        \prod_{i=1}^{D}\frac{1}
                        {\sinh\bigl(\pi\gamma \omega\w_i^{\!\top}\bm{\beta}\bigr)}\,d\omega \Bigr|\, d t \label{eq:expand2}
\end{align}
Now, integrating \eqnref{eq:expand1} with respect to $\bm{\beta}$ yields the $\tv$ of $f_{\sf{DT},\gamma}$,

\begin{align*}
     \rtv{f_{\sf{DT}, \gamma}} = c_D \int_{\mathbb{S}^{d-1}}\int_{\reals} \Bigl| \int_{\mathbb R}
        (\omega)^{d+1}\,
        e^{\mathrm i\omega \paren{t - \sum_{i=1}^D b_i \w_i^{\top} \bm{\beta}}}\,
        \prod_{i=1}^{D}\frac{1}
                        {\sinh\bigl(\pi\gamma \omega\w_i^{\!\top}\bm{\beta}\bigr)}\,d\omega \Bigr|\, d t\, d \bm{\beta}
\end{align*}
where $c_D := (2\pi)^{-(D+1)/2} (\gamma\pi)^D$.


Finally, we will show that the RHS obtained in the form of $\rtv{f_{\sf{DT}, \gamma}}$ above is unbounded.

Fix $\omega_0>0$ as above. For each $\beta\in\mathbb S^{d-1}$ define
\[
u_\beta(t):=\partial_t^{\,d+1}\mathcal R\{f_{\sf DT,\gamma}\}(\beta,t).
\]
Let $\widehat\phi(\omega):=\mathbf 1_{\{0 \le  \omega\le \omega_0\}}$ and let
$\phi:=\mathcal F_t^{-1}[\widehat\phi]$ under our Fourier convention.
Then
\[ \|\phi\|_\infty \le (2\pi)^{-1/2}\int_0^{\omega_0} d\omega
= \frac{\omega_0}{\sqrt{2\pi}}.
\]
Fix the shift
\[
t_0(\beta):=\sum_{i=1}^D b_i\,w_i^\top\beta
\]
and define $\phi_\beta(t):=\phi(t-t_0(\beta))$. By Hölder,
\[
\|u_\beta\|_{L^1_t}
\;\ge\;
\frac{1}{\|\phi\|_\infty}\Bigl|\int_{\mathbb R} u_\beta(t)\,\phi_\beta(t)\,dt\Bigr|.
\]
Using Plancherel/Parseval for our Fourier convention and $\widehat{\phi_\beta}(\omega)
=e^{-i\omega t_0(\beta)}\widehat\phi(\omega)$, we get
\[\int_{\mathbb R} u_\beta(t)\phi_\beta(t)\,dt
=(2\pi)^{-1/2}\int_{0}^{\omega_0} (i\omega)^{d+1}e^{-i\omega t_0(\beta)}
\,\widehat f_{\sf DT,\gamma}(\omega\beta)\,d\omega.
\]
Plugging in the expression for $\widehat f_{\sf DT,\gamma}$ and using the cancellation of the
phase $e^{-i\omega t_0(\beta)}$ yields (for $\sigma$-a.e.\ $\beta$)
\[
\Bigl|\int_{\mathbb R} u_\beta(t)\,\phi_\beta(t)\,dt\Bigr|
\;\ge\;
C_0\Bigl|\int_{0}^{\omega_0}\omega^{d+1}
\prod_{i=1}^D \frac{1}{\sinh(\pi\gamma\omega\,w_i^\top\beta)}\,d\omega\Bigr|.
\]


On $(0,\omega_0)$ we have $|\pi\gamma\omega\,w_i^\top\beta|\le 1$, hence for $|z|\le 1$,
$|\sinh(z)|\le 2|z|$ and so
\[
\frac{1}{|\sinh(\pi\gamma\omega\,w_i^\top\beta)|}\ge \frac{1}{2\pi\gamma\omega\,|w_i^\top\beta|}.
\]
Therefore, for $\omega\in(0,\omega_0)$,
\[
\prod_{i=1}^D \frac{1}{|\sinh(\pi\gamma\omega\,w_i^\top\beta)|}
\ge
\frac{1}{(2\pi\gamma)^D}\cdot \frac{1}{\omega^D}\cdot \frac{1}{|w_1^\top\beta|},
\]
where we used $|w_i^\top\beta|\le 1$ for $i\ge 2$ after normalizing $\|w_i\|_2=1$
(or absorb $\max_i\|w_i\|_2$ into constants).
Hence, 
\[\Bigl|\int_{0}^{\omega_0}\omega^{d+1}
\prod_{i=1}^D \frac{1}{\sinh(\pi\gamma\omega\,w_i^\top\beta)}\,d\omega\Bigr|
=\int_{0}^{\omega_0}\Bigl|\omega^{d+1}
\prod_{i=1}^D \frac{1}{\sinh(\pi\gamma\omega\,w_i^\top\beta)}\Bigr|\,d\omega,
\]
If $D\le d+1$, the $\omega$-integral equals a finite constant $C_2>0$; if $D\ge d+2$
it diverges already in $\omega$. In either case (for $d\ge 2$),
\[
\|u_\beta\|_{L^1_t}\;\ge\;\frac{C_3}{|w_1^\top\beta|}.
\]
Integrating over $\beta\in\mathbb S^{d-1}$ gives
\[
\rtv{f_{\sf DT,\gamma}}
\;=\;c_d\int_{\mathbb S^{d-1}}\|u_\beta\|_{L^1_t}\,d\sigma(\beta)
\;\ge\;C_4\int_{\mathbb S^{d-1}}\frac{1}{|w_1^\top\beta|}\,d\sigma(\beta)
\;=\;+\infty,
\]
since $w_1^\top\beta=0$ defines an equator and $|\beta_1|^{-1}$ is not integrable near $0$ on
$\mathbb S^{d-1}$ when $d\ge 2$.

\newpage




%% file: app_Gaussian.tex
\section{Gaussian Smoothing}\label{app:gauss}

In this appendix, we provide the proof of the main theorem on the approximation of Gaussian smoothed decision trees as stated in \thmref{thm:gaussian}.

\subsection{Upper bounding the $\tv$ for the Gaussain-smoothed decision trees}

Let \( f_{\sf{DT}}(\cdot) = {1}\curly{\cdot \in A}\) be the indicator function of a decision-tree region \( A \subset \mathbb{R}^d \), and define the smoothed function via convolution:
\[
f_\sigma(\x) := (f_{\sf{DT}} * G_\sigma)(\x) = \int_{\mathbb{R}^d} {1}\curly{\y \in A}\, G_\sigma(\x - \y)\, d\y,
\]
where \( G_\sigma(\z) := \frac{1}{(2\pi \sigma^2)^{d/2}} \exp\left(-\frac{\|\z\|^2}{2\sigma^2}\right) \) is the Gaussian kernel.

By the Fourier-slice formula,
\[
R\{f_\sigma\}(\bm{\beta}, t) = \frac{1}{\sqrt{2\pi}} \int_{\mathbb{R}} e^{\mathrm i\omega t} \widehat{f_\sigma}(\omega \bm{\beta})\, d\omega,
\]
and since \( \widehat{f_\sigma}(\bm{\xi}) = \widehat{f_{\sf{DT}}}(\bm{\xi})\, e^{-\frac{\sigma^2 \|\bm{\xi}\|^2}{2}} \) (using Convolution theorem of Fourier transform), we compute:
\begin{align}
    \partial_t^{d+1} R\{f_\sigma\}(\bm{\beta}, t)
= \frac{1}{\sqrt{2\pi}} \int_{\mathbb{R}} (\mathrm i\omega)^{d+1} e^{\mathrm i\omega t} \widehat{f_{\sf{DT}}}(\omega \bm{\beta})\, e^{-\frac{\sigma^2 \omega^2}{2}} d\omega.\label{eq:Gd+1}
\end{align}

As shown in \eqnref{eq:fourDT}, we have:
\[
\widehat{f_{\sf{DT}}}(\omega \bm{\beta})
= \frac{1}{(2\pi)^{d/2}} \int_A e^{-\mathrm i\omega \bm{\beta}^\top \x} d\x.
\]
Plugging this equation in \eqnref{eq:Gd+1}, we get:
\begin{align}
    \partial_t^{d+1} R\{f_\sigma\}(\bm{\beta}, t)
= \frac{1}{(2\pi)^{(d+1)/2}} \int_A \underbrace{\left[ \int_{\mathbb{R}} (\mathrm i\omega)^{d+1} e^{\mathrm i\omega(t - \bm{\beta}^\top \x)} e^{-\frac{\sigma^2 \omega^2}{2}} d\omega \right]}_{I_{d+1}} d\x.    \label{eq:d+1}
\end{align}
In the following, we would rewrite $I_{d+1}$ in terms of probabilist's Hermite polynomials. 

Set \(m=d+1\) and \(s := t-\bm{\beta}^\top \x\).  Now, using Rodrigues' formula for Hermite polynomials: ${\mathrm{He}_{n}(u)=(-1)^{n}e^{\frac {u^{2}}{2}}{\frac {d^{n}}{du^{n}}}e^{-{\frac {u^{2}}{2}}}}$, and integration by parts, we have
\begin{align}
    I_m = \int_{\mathbb R} (\mathrm i\omega)^{m}\,e^{\,\mathrm i\omega s}\,e^{-\frac{\sigma^{2}\omega^{2}}{2}}\,d\omega
   = (- \mathrm i)^{m}\,\frac{\sqrt{2\pi}}{\sigma^{m+1}}\;
     e^{-\frac{s^{2}}{2\sigma^{2}}}\;
     \mathrm{He}_{m}\!\Bigl(\frac{s}{\sigma}\Bigr).    
\end{align}

With this, we can write:
\begin{align}
    \left|\partial_t^{d+1} R\{f_\sigma\}(\bm{\beta}, t)\right|
&= \frac{1}{(2\pi)^{(d+1)/2}} \Bigl|\int_A (- \mathrm i)^{d+1}\,\frac{\sqrt{2\pi}}{\sigma^{d+2}}\;
     e^{-\frac{(t-\bm{\beta}^\top \x)^{2}}{2\sigma^{2}}}\;
     \mathrm{He}_{d+1}\!\Bigl(\frac{t-\bm{\beta}^\top \x}{\sigma}\Bigr) d\x\Bigr| 
\end{align}
Hence, we can write
\begin{align}
\left|\partial_t^{\,d+1} R\{f_\sigma\}(\bm{\beta},t)\right|
   \le \frac{\sigma^{-(d+2)}}{(2\pi)^{d/2}}
          \int_{A}
          \Bigl|\mathrm{He}_{d+1}\!\Bigl(\tfrac{t-\bm{\beta}^{\!\top}\x}{\sigma}\Bigr)\Bigr|
          \exp\!\Bigl(-\tfrac{(t-\bm{\beta}^{\!\top}\x)^2}{2\sigma^{2}}\Bigr)\,d\x.    \label{eq:boundG}
\end{align}
In odd dimensions, $\tv$ of $f_\sigma$ is the $L^1$-norm of the absolute value of $(d+1)-$th derivative of $R\{f_\sigma\}(\bm{\beta},t)$ with respect to $t$. Thus, we get
\allowdisplaybreaks
\begin{align*}
    ||f_{\sigma}||_\cR &= \int_{\mathbb{S}^{d-1}}\int_{\reals} \left|\partial_t^{\,d+1} R\{f_\sigma\}(\bm{\beta},t)\right| d t\, d \bm{\beta}\\
         & \le \frac{\sigma^{-(d+2)}}{(2\pi)^{d/2}} \int_{\mathbb{S}^{d-1}}\int_{\reals} \int_{A}
          \Bigl|\mathrm{He}_{d+1}\!\Bigl(\tfrac{t-\bm{\beta}^{\!\top}\x}{\sigma}\Bigr)\Bigr|
          \exp\!\Bigl(-\tfrac{(t-\bm{\beta}^{\!\top}\x)^2}{2\sigma^{2}}\Bigr)\,d\x\, dt\, d \bm{\beta}\\
          &= \frac{\sigma^{-(d+2)}}{(2\pi)^{d/2}} \int_{\mathbb{S}^{d-1}} \int_A \int_{\reals} \Bigl|\mathrm{He}_{d+1}\!\Bigl(\tfrac{t-\bm{\beta}^{\!\top}\x}{\sigma}\Bigr)\Bigr|
          \exp\!\Bigl(-\tfrac{(t-\bm{\beta}^{\!\top}\x)^2}{2\sigma^{2}}\Bigr)\,dt\, d\x\, d \bm{\beta}
\end{align*}

Substituting \( u = \frac{t - \bm{\beta}^\top \x}{\sigma} \), \( dt = \sigma du \), and noting that this inner integral is independent of \( x \), we obtain:
\begin{align}
    ||f_{\sigma}||_\cR \le \frac{\sigma^{-(d+1)}}{(2\pi)^{d/2}}  \int_{\mathbb{S}^{d-1}} \int_A C_{\mathrm{He}}(d+1) d\x \, d\bm{\beta} \label{eq:GRbound}
\end{align}
where we define
\[
C_{\mathrm{He}}(m)\;:=\;\int_{\mathbb R}
        e^{-u^{2}/2}\,
        \bigl|\mathrm{He}_{m}(u)\bigr|\,du.
\]

Simplifying the integrals with respect to $t$ and $\bm{\beta}$, we can rewrite \eqnref{eq:GRbound} as
\begin{align}
     ||f_{\sigma}||_\cR \le \frac{\sigma^{-(d + 1)}}{(2\pi)^{d/2}}C_{\mathrm{He}}(d+1)\cdot  Vol(A) \cdot |\mathbb{S}^{d-1}| \label{eq:RG1}
\end{align}
Now, we will bound $C_{\mathrm{He}}(d+1)$. Note that 
\begin{align*}
    C_{\mathrm{He}}(m)\;=\;\int_{\mathbb R}
        e^{-u^{2}/2}\,
        \bigl|\mathrm{He}_{m}(u)\bigr|\,du \le \sqrt{\paren{\int_\reals e^{-u^{2}/2}\, du} \paren{\int_\reals e^{-u^{2}/2}\,
        \mathrm{He}^2_{m}(u)\, du} } \quad (\star)
\end{align*}
where we have used Cauchy-Schwarz in the last inequality.

Note that the first term in $(\star)$ has a concrete form due to integral of a Gaussian density, and hence
\begin{align*}
    \int_\reals e^{-u^{2}/2}\, du = \sqrt{2 \pi}
\end{align*}
In the second term of the $(\star)$, we use a standard identity on inner product of probabilit's Hermite polynomials, 
\begin{align*}
    \int_\reals e^{-u^{2}/2}\,
        \mathrm{He}^2_{m}(u)\, du = \sqrt{2\pi} m!
\end{align*}
Using the approximations above, we can bound 
\begin{align*}
    C_{\mathrm{He}}(d+1) \le \sqrt{2\pi} \sqrt{(d+1)!} 
\end{align*}

Now, using Sterling's approximation, we can further simplify the rhs as
\[%
C_{\mathrm{He}}(d{+}1)\;\le\;
c\,(d{+}1)^{\frac{d+1}{2}}\;2^{d/2}
\qquad(\text{some universal }c\approx 1.2).
\]
for $d \ge 3$.

Using this bound to simplify \eqnref{eq:RG1}, we get 
\begin{align*}
    ||f_{\sigma}||_\cR &\le \frac{\sigma^{-(d+1)}}{(2\pi)^{d/2}} \cdot \frac{2\pi^{d/2}}{\Gamma(d/2)} \cdot c\,(d{+}1)^{\frac{d+1}{2}}\;2^{d/2} \cdot Vol(A) \\
    & = \frac{\sigma^{-(d+1)}}{\Gamma(d/2)} \cdot c\,(d{+}1)^{\frac{d+1}{2}} \cdot Vol(A) 
\end{align*}
Now, using the Stirling's approximation on \(\Gamma(d/2) \approx \sqrt{2\pi} \left(\frac{d}{2e}\right)^{d/2} \sqrt{\frac{2}{d}}\):
we have
\[%
\|f_\sigma\|_{R}
   \;\le\;
   C\,d^{\,1/2}\;
   \Bigl(\tfrac{\sqrt{2}\,e}{\sigma}\Bigr)^{\!d}\;
   Vol(A)
\quad(\sigma>0),
\]
where \(C \le 2.2\).

\newpage

%% file: app_class.tex
\section{Approximation Post-thresholding}\label{app:post-thres}

In this section we prove \lemref{lem:single-box-main} and \thmref{thm:RTV-single-main} from \secref{sec:easy}.

First we state a technical lemma that constructs smooth monotone barriers with exact plateaus and controlled derivatives.

\begin{lemma}\label{lem:barrier-fixed}
Fix integer $d\ge1$, scalars $\lambda\ge1$, and $\varepsilon\in(0,1]$.
Let $H\in C^\infty(\mathbb R)$ be nondecreasing with
\[
H(s)=0\ (s\le 0),\qquad H(s)=1\ (s\ge 1),\qquad H^{(m)}(0)=H^{(m)}(1)=0\ \ (1\le m\le d{+}1).
\]
Define the scaled step $h_\varepsilon(t):=H\!\big((t+\varepsilon)/\varepsilon\big)$ such that $h_\varepsilon=0$ on $(-\infty,-\varepsilon]$, $h_\varepsilon=1$ on $[0,\infty)$, and $h_\varepsilon'\ge0$ supported in $[-\varepsilon,0]$, and set
\[
\vartheta_{\lambda,\varepsilon}(t)\ :=\ (1-h_\varepsilon(t))\,e^{\lambda t}\ +\ h_\varepsilon(t).
\]
Then $\vartheta_{\lambda,\varepsilon}\in C^{d+1}(\mathbb R)$, is nondecreasing, and
\[
\vartheta_{\lambda,\varepsilon}(t)=e^{\lambda t}\quad (t\le -\varepsilon),\qquad
\vartheta_{\lambda,\varepsilon}(t)=1\quad (t\ge 0).
\]
For every $q\in\{1,\dots,d{+}1\}$ there exist constants $C_q$ (depending only on $q$ and $H$) such that
\[
\|\vartheta_{\lambda,\varepsilon}^{(q)}\|_{L^\infty(\mathbb R)}\ \le\ C_q\sum_{m=0}^q \lambda^{\,q-m}\,\varepsilon^{-m},\qquad
\|\vartheta_{\lambda,\varepsilon}^{(q)}\|_{L^1(\mathbb R)}\ \le\ C_q\Big(\lambda^{q-1}+\sum_{m=1}^q \lambda^{\,q-m}\,\varepsilon^{\,1-m}\Big).
\]
In particular, for $\varepsilon=c_0/\lambda$ (fixed $c_0>0$),
\[
\|\vartheta_{\lambda,\varepsilon}^{(q)}\|_{L^\infty}\ \le\ C'_q\,\lambda^q,\qquad
\|\vartheta_{\lambda,\varepsilon}^{(q)}\|_{L^1}\ \le\ C'_q\,\lambda^{q-1}.
\]
\end{lemma}
\begin{proof}[Proof of Lemma~\ref{lem:barrier-fixed}]
Let
\[
h_\varepsilon(t):=H\!\left(\frac{t+\varepsilon}{\varepsilon}\right),\qquad
\vartheta(t):=\vartheta_{\lambda,\varepsilon}(t)=(1-h_\varepsilon(t))e^{\lambda t}+h_\varepsilon(t).
\]

Since $H(s)=0$ for $s\le 0$ and $H(s)=1$ for $s\ge 1$, we have
\[
h_\varepsilon(t)=0\ \ (t\le -\varepsilon),\qquad
h_\varepsilon(t)=1\ \ (t\ge 0).
\]
Moreover, $H$ is nondecreasing, hence $h_\varepsilon$ is nondecreasing and $h_\varepsilon'(t)\ge 0$.
For every integer $m\ge 1$, $h_\varepsilon^{(m)}$ is supported in $[-\varepsilon,0]$.

Because the inner map $t\mapsto (t+\varepsilon)/\varepsilon$ is affine, repeated chain rule gives the exact identity
\begin{equation}\label{eq:heps-deriv}
h_\varepsilon^{(m)}(t)=\varepsilon^{-m}\,H^{(m)}\!\left(\frac{t+\varepsilon}{\varepsilon}\right),\qquad m\ge 0.
\end{equation}
In particular, for each $m\ge 0$,
\[
\|h_\varepsilon^{(m)}\|_{L^\infty(\mathbb R)}
\le \varepsilon^{-m}\,\|H^{(m)}\|_{L^\infty(\mathbb R)}
=:C_m\,\varepsilon^{-m}.
\]
For $m\ge 1$, using $\mathrm{supp}(h_\varepsilon^{(m)})\subset[-\varepsilon,0]$ and the change of variables
$s=(t+\varepsilon)/\varepsilon$ (so $dt=\varepsilon\,ds$), we also get
\[
\|h_\varepsilon^{(m)}\|_{L^1(\mathbb R)}
=\int_{-\varepsilon}^0 \varepsilon^{-m}\left|H^{(m)}\!\left(\frac{t+\varepsilon}{\varepsilon}\right)\right|dt
=\varepsilon^{1-m}\int_0^1 |H^{(m)}(s)|\,ds
=:\widetilde C_m\,\varepsilon^{1-m}.
\]
Here $C_m$ and $\widetilde C_m$ depend only on $m$ and $H$, not on $\lambda$ or $\varepsilon$.

Since $h_\varepsilon$ and $e^{\lambda t}$ are smooth, $\vartheta\in C^\infty(\mathbb R)$ (hence in $C^{d+1}$).
If $t\le -\varepsilon$ then $h_\varepsilon(t)=0$ so $\vartheta(t)=e^{\lambda t}$.
If $t\ge 0$ then $h_\varepsilon(t)=1$ so $\vartheta(t)=1$.

Now differentiate $\vartheta$:
\[
\vartheta'(t)=-(h_\varepsilon'(t))e^{\lambda t}+(1-h_\varepsilon(t))\lambda e^{\lambda t}+h_\varepsilon'(t)
=(1-h_\varepsilon(t))\lambda e^{\lambda t}+h_\varepsilon'(t)\bigl(1-e^{\lambda t}\bigr).
\]
On $[-\varepsilon,0]$ we have $h_\varepsilon'\ge 0$ and $t\le 0$ implies $1-e^{\lambda t}\ge 0$, hence $\vartheta'(t)\ge 0$.
On $(-\infty,-\varepsilon]$ we have $\vartheta(t)=e^{\lambda t}$ which is increasing, and on $[0,\infty)$ we have $\vartheta\equiv 1$.
Therefore $\vartheta$ is nondecreasing on $\mathbb R$.

Fix $q\in\{1,\dots,d+1\}$. Write
\[
\vartheta(t)=(1-h_\varepsilon(t))e^{\lambda t}+h_\varepsilon(t).
\]
By Leibniz' rule,
\begin{equation}\label{eq:vartheta-q}
\vartheta^{(q)}(t)
=\sum_{m=0}^q \binom{q}{m}(1-h_\varepsilon)^{(m)}(t)\,\partial_t^{\,q-m}(e^{\lambda t})
+h_\varepsilon^{(q)}(t).
\end{equation}
Since $\partial_t^{\,q-m}(e^{\lambda t})=\lambda^{q-m}e^{\lambda t}$ and $(1-h_\varepsilon)^{(m)}=-h_\varepsilon^{(m)}$ for $m\ge 1$,
we may bound pointwise (using $0\le 1-h_\varepsilon\le 1$):
\[
|\vartheta^{(q)}(t)|
\le \lambda^q e^{\lambda t}
+\sum_{m=1}^q \binom{q}{m}\lambda^{q-m}e^{\lambda t}\,|h_\varepsilon^{(m)}(t)|
+|h_\varepsilon^{(q)}(t)|.
\]

\underbar{\emph{$L^\infty$ bound.}}
On $t\le 0$ we have $e^{\lambda t}\le 1$, so using $\|h_\varepsilon^{(m)}\|_\infty\le C_m\varepsilon^{-m}$ gives
\[
\sup_{t\le 0}|\vartheta^{(q)}(t)|
\le \lambda^q+\sum_{m=1}^q \binom{q}{m}\lambda^{q-m}C_m\varepsilon^{-m}+C_q\varepsilon^{-q}.
\]
On $t\ge 0$, $\vartheta\equiv 1$ so $\vartheta^{(q)}\equiv 0$.
Absorbing binomial coefficients into a constant $C_q$ (depending only on $q$ and $H$) yields
\[
\|\vartheta^{(q)}\|_{L^\infty(\mathbb R)}
\le C_q\sum_{m=0}^q \lambda^{q-m}\varepsilon^{-m}.
\]

\underbar{\emph{$L^1$ bound.}}
Split $\mathbb R$ into $(-\infty,-\varepsilon]\cup[-\varepsilon,0]\cup[0,\infty)$.

On $(-\infty,-\varepsilon]$, $\vartheta^{(q)}(t)=\lambda^q e^{\lambda t}$, hence
\[
\int_{-\infty}^{-\varepsilon}|\vartheta^{(q)}(t)|\,dt
=\int_{-\infty}^{-\varepsilon}\lambda^q e^{\lambda t}\,dt
=\lambda^{q-1}e^{-\lambda\varepsilon}
\le \lambda^{q-1}.
\]
On $[0,\infty)$, $\vartheta^{(q)}\equiv 0$.
On $[-\varepsilon,0]$ we have $e^{\lambda t}\le 1$, so integrating the pointwise bound and using
$\|h_\varepsilon^{(m)}\|_{L^1}\le \widetilde C_m\varepsilon^{1-m}$ gives
\[
\int_{-\varepsilon}^{0}|\vartheta^{(q)}(t)|\,dt
\le \lambda^q\varepsilon
+\sum_{m=1}^q \binom{q}{m}\lambda^{q-m}\widetilde C_m\varepsilon^{1-m}
+\widetilde C_q\varepsilon^{1-q}.
\]
Absorbing constants again into $C_q$ yields
\[
\|\vartheta^{(q)}\|_{L^1(\mathbb R)}
\le C_q\Big(\lambda^{q-1}+\sum_{m=1}^q \lambda^{q-m}\varepsilon^{1-m}\Big).
\]

If $\varepsilon=c_0/\lambda$, then
\[
\sum_{m=0}^q \lambda^{q-m}\varepsilon^{-m}
=\sum_{m=0}^q \lambda^{q-m}\left(\frac{\lambda}{c_0}\right)^m
=\lambda^q\sum_{m=0}^q c_0^{-m}
\le C'_q\,\lambda^q,
\]
and similarly
\[
\lambda^{q-1}+\sum_{m=1}^q \lambda^{q-m}\varepsilon^{1-m}
=\lambda^{q-1}\Big(1+\sum_{m=1}^q c_0^{1-m}\Big)
\le C'_q\,\lambda^{q-1}.
\]
This proves the stated bounds.
\end{proof}

Using this barrier, we now prove \lemref{lem:single-box-main}.

\begin{proof}[Proof of Lemma~\ref{lem:single-box-main}]
Recall that
\[
S_B(x)\ :=\ \prod_{j=1}^d \vartheta_{\lambda,\varepsilon}(u_j-x_j)\,
\vartheta_{\lambda,\varepsilon}(x_j-\ell_j),\qquad \varepsilon=c_0/\lambda.
\]
We use the barrier properties from Lemma~\ref{lem:barrier-fixed}:
\[
0\le \vartheta_{\lambda,\varepsilon}(t)\le 1,\qquad
\vartheta_{\lambda,\varepsilon}(t)=1\ (t\ge 0),\qquad
\vartheta_{\lambda,\varepsilon}(t)=e^{\lambda t}\ (t\le -\varepsilon).
\]

\emph{Exact thresholding.}
If $x\in B$, then $u_j-x_j\ge 0$ and $x_j-\ell_j\ge 0$ for all $j$, hence every factor equals $1$ and
$S_B(x)=1$. If $x\notin B$, then for some $j$ either $u_j-x_j<0$ or $x_j-\ell_j<0$; by monotonicity
$\vartheta_{\lambda,\varepsilon}(\cdot)<1$ on $(-\infty,0)$, so at least one factor is $<1$, hence $S_B(x)<1$.
Therefore $\{x:\ S_B(x)\ge 1\}=B$.

\emph{$L^1(P)$ closeness.}
Since $S_B=\mathbf 1_B=1$ on $B$ and $0\le S_B\le 1$ everywhere,
\[
\mathbb E\big[|S_B(X)-\mathbf 1_B(X)|\big]
=\mathbb E\big[S_B(X)\mathbf 1_{B^c}(X)\big].
\]

For $x\in\mathbb R^d$, define the coordinate overhangs
\[
\delta_j(x):=(\ell_j-x_j)_+ + (x_j-u_j)_+ \quad(\ge 0),\qquad
Z(x):=\sum_{j=1}^d \delta_j(x)=d_1(x,B).
\]
If $x\notin B$, then for each $j$ we have $\vartheta_{\lambda,\varepsilon}(u_j-x_j)\vartheta_{\lambda,\varepsilon}(x_j-\ell_j) = \vartheta_{\lambda,\varepsilon}(- \delta_j(x))$
, hence by monotonicity
\[
S_B(x)\le \prod_{j=1}^d \vartheta_{\lambda,\varepsilon}(-\delta_j(x)).
\]
Using the exponential form of $\vartheta_{\lambda,\varepsilon}$, for every $\delta\ge 0$,
\[
\vartheta_{\lambda,\varepsilon}(-\delta)\ \le\ e^{-\lambda(\delta-\varepsilon)_+},
\]
because if $\delta\ge \varepsilon$ then $-\delta\le -\varepsilon$ and $\vartheta_{\lambda,\varepsilon}(-\delta)=e^{-\lambda\delta}
\le e^{-\lambda(\delta-\varepsilon)}$, while if $\delta\le \varepsilon$ then $\vartheta_{\lambda,\varepsilon}(-\delta)\le 1
=e^{-\lambda(\delta-\varepsilon)_+}$.
Therefore for $x\notin B$,
\[
S_B(x)\le \exp\!\Big(-\lambda\sum_{j=1}^d(\delta_j(x)-\varepsilon)_+\Big).
\]
Since $\sum_{j=1}^d(\delta_j-\varepsilon)_+ \ge (\sum_j\delta_j-d\varepsilon)_+=(Z-d\varepsilon)_+$,
\[
S_B(x)\le \exp\!\big(-\lambda(Z(x)-d\varepsilon)_+\big)
\le \mathbf 1\{0<Z(x)\le d\varepsilon\}+e^{\lambda d\varepsilon}e^{-\lambda Z(x)}\mathbf 1\{Z(x)>0\}.
\]
Multiplying by $\mathbf 1_{B^c}(x)$ (equivalently $\mathbf 1\{Z(x)>0\}$) and taking expectations gives
\begin{equation}\label{eq:SB_split}
\mathbb E[S_B(X)\mathbf 1_{B^c}(X)]
\le \mathbb P\{0<Z(X)\le d\varepsilon\} + e^{\lambda d\varepsilon}\,\mathbb E[e^{-\lambda Z(X)}\mathbf 1_{B^c}(X)].
\end{equation}

Next we relate $Z$ to Euclidean distance to the boundary.

Let $D(x):=\mathrm{dist}(x,\partial B)$.
On $B^c$, the closest point in $B$ lies on $\partial B$, hence $\mathrm{dist}(x,B)=D(x)$; moreover $Z(x)=d_1(x,B)\ge \mathrm{dist}(x,B)=D(x)$.
Thus, on $B^c$, $\{0<Z\le t\}\subseteq\{D\le t\}$ and $e^{-\lambda Z}\le e^{-\lambda D}$, so
\[
\mathbb P\{0<Z(X)\le t\}\le \mathbb P\{D(X)\le t\},\qquad
\mathbb E[e^{-\lambda Z(X)}\mathbf 1_{B^c}(X)]\le \mathbb E[e^{-\lambda D(X)}\mathbf 1_{B^c}(X)].
\]

Let $t_0>0$ be such that the tube--mass condition holds for all $t\in(0,t_0]$:
$\mathbb P\{D(X)\le t\}\le C t^\beta$.
Applying this with $t=d\varepsilon$ yields
\[
\mathbb P\{0<Z(X)\le d\varepsilon\}\le \mathbb P\{D(X)\le d\varepsilon\}\le C(d\varepsilon)^\beta
= C(d c_0)^\beta \lambda^{-\beta}.
\]

For the Laplace term, use the layer-cake/integration-by-parts bound for the nonnegative random variable $D\mathbf 1_{B^c}$:
\[
\mathbb E[e^{-\lambda D(X)}\mathbf 1_{B^c}(X)]
=\int_{0}^{\infty} e^{-\lambda t}\,d\mathbb P\{0<D(X)\le t\}
\le \lambda\int_{0}^{t_0} e^{-\lambda t}\,\mathbb P\{D(X)\le t\}\,dt + e^{-\lambda t_0}.
\]
Using $\mathbb P\{D(X)\le t\}\le Ct^\beta$ for $t\le t_0$ gives
\[
\mathbb E[e^{-\lambda D(X)}\mathbf 1_{B^c}(X)]
\le C\lambda\int_0^\infty e^{-\lambda t}t^\beta\,dt + e^{-\lambda t_0}
= C\Gamma(\beta+1)\lambda^{-\beta} + e^{-\lambda t_0}.
\]
Since $\lambda\ge 1$, the exponential tail satisfies $e^{-\lambda t_0}\le C_{\beta,t_0}\lambda^{-\beta}$, hence
\[
\mathbb E[e^{-\lambda D(X)}\mathbf 1_{B^c}(X)]
\le C'_{\beta,t_0}\, \lambda^{-\beta}.
\]
Plugging these estimates into \eqref{eq:SB_split} and using $e^{\lambda d\varepsilon}=e^{dc_0}$ yields
\[
\mathbb E[|S_B(X)-\mathbf 1_B(X)|]
=\mathbb E[S_B(X)\mathbf 1_{B^c}(X)]
\le \Big(C(d c_0)^\beta + e^{dc_0}C'_{\beta,t_0}\Big)\lambda^{-\beta},
\]
which is the claimed $O(\lambda^{-\beta})$ bound with $C_{d,\beta,c_0}$ absorbing constants.
\end{proof}

\subsection{Proof of upper bound on the $\cR\mathrm{TV}$}
Now, we show the proof of the upper bound on the Radon total-variation as stated in \thmref{thm:RTV-single-main} for a single box from \secref{sec:easy}.

First, we estblish some auxiliary lemmas about the 1D barrier function $\vartheta_{\lambda,\varepsilon}$, and its use in constructing box indicators with controlled derivatives.

\begin{lemma}\label{lem:prod-1d-fixed}
Assume $\varepsilon=c_0/\lambda$ with fixed $c_0>0$.
Let $U(t)=\vartheta_{\lambda,\varepsilon}(u-t)$, $L(t)=\vartheta_{\lambda,\varepsilon}(t-\ell)$ with $u>\ell$.
For $q\ge0$ set $F_q:=\partial_t^{\,q}(U L)$. Then
\[
\|F_0\|_{L^1(\mathbb R)}\ \le\ (u-\ell)\ +\ C\,\tfrac{1}{\lambda},\qquad
\|F_q\|_{L^1(\mathbb R)}\ \le\ C_q\,\lambda^{\,q-1}\quad(q\ge1).
\]
\end{lemma}
\begin{proof}
Let $\vartheta=\vartheta_{\lambda,\varepsilon}$. Recall $\vartheta(t)=1$ for $t\ge 0$ and
$\vartheta^{(m)}(t)=0$ for all $m\ge 1$ and all $t\ge 0$ (exact plateau).

On $[\ell,u]$ we have $t-\ell\ge 0$ and $u-t\ge 0$, hence $L(t)=U(t)=1$ and $UL=1$.
On the transition layers $[\ell-\varepsilon,\ell]\cup[u,u+\varepsilon]$ we have $0\le UL\le 1$, so their contribution is at most $2\varepsilon$.
On the left tail $(-\infty,\ell-\varepsilon]$ we have $L(t)=e^{\lambda(t-\ell)}$ and $U(t)=1$ (since $t\le \ell<u$),
so $\int_{-\infty}^{\ell-\varepsilon} UL\,dt \le \int_{-\infty}^{\ell-\varepsilon} e^{\lambda(t-\ell)}dt = \lambda^{-1}e^{-\lambda\varepsilon}\le \lambda^{-1}$.
Similarly the right tail contributes at most $\lambda^{-1}$.
Thus
\[
\|F_0\|_{L^1(\mathbb R)}=\int_\mathbb R UL
\le (u-\ell)+2\varepsilon+\frac{2}{\lambda}.
\]

By Leibniz,
\[
\partial_t^{\,q}(UL)=\sum_{a=0}^q \binom{q}{a} U^{(a)} L^{(q-a)}.
\]
Crucially, for any $a\ge 1$ we have
\[
U^{(a)}(t)=(-1)^a \vartheta^{(a)}(u-t)=0\quad\text{whenever }u-t\ge 0 \iff t\le u,
\]
so $\mathrm{supp}(U^{(a)})\subset (u,\infty)$.
Likewise for any $b\ge 1$,
\[
L^{(b)}(t)=\vartheta^{(b)}(t-\ell)=0\quad\text{whenever }t-\ell\ge 0 \iff t\ge \ell,
\]
so $\mathrm{supp}(L^{(b)})\subset (-\infty,\ell)$.
Since $\ell<u$, these supports are disjoint, hence for $1\le a\le q-1$ we have
$U^{(a)}(t)\,L^{(q-a)}(t)\equiv 0$ for all $t$.
Therefore only the endpoint terms remain:
\[
\partial_t^{\,q}(UL)=U^{(q)}L+U L^{(q)}.
\]
Using $0\le U,L\le 1$ we get
\[
\|F_q\|_{L^1}\le \|U^{(q)}\|_{L^1}+\|L^{(q)}\|_{L^1}
= 2\|\vartheta^{(q)}\|_{L^1(\mathbb R)}.
\]
Now apply Lemma~\ref{lem:barrier-fixed} to $\vartheta^{(q)}$.
If $\varepsilon=c_0/\lambda$, this gives $\|\vartheta^{(q)}\|_{L^1}\le C_q \lambda^{q-1}$, hence
$\|F_q\|_{L^1}\le C_q' \lambda^{q-1}$.
\end{proof}

\begin{lemma}\label{lem:per-box-Lp}
Assume $\varepsilon=c_0/\lambda$ with fixed $c_0>0$.
Let $B=\prod_{j=1}^d[\ell_j,u_j]$ and $S_B(x)=\prod_{j=1}^d G_j(x_j)$ with
$G_j(t)=\vartheta_{\lambda,\varepsilon}(u_j-t)\,\vartheta_{\lambda,\varepsilon}(t-\ell_j)$.
Fix $1\le p\le\infty$ and a multiindex $\alpha$ with $|\alpha|=s\ge1$, and set $q_j=\alpha_j$,
$J=\{j:q_j\ge1\}$ and $r=|J|$. Then
\[
\|\partial^\alpha S_B\|_{L^p(\mathbb R^d)}
\ \le\ C_{d,s,p}\,
\lambda^{\,s-\frac{r}{p}}\,
\prod_{k\notin J}\Big((u_k-\ell_k)+\tfrac{C}{\lambda}\Big)^{\frac1p}.
\]
\end{lemma}

\begin{proof}
By separability (ignoring harmless signs), $\partial^\alpha S_B(x)=\prod_{j=1}^d \partial^{q_j}G_j(x_j)$.
If $1\le p<\infty$, Tonelli gives
\[
\|\partial^\alpha S_B\|_{L^p}^p=\prod_{j=1}^d \|\partial^{q_j}G_j\|_{L^p(\mathbb R)}^p,
\]
while for $p=\infty$ we have $\|\partial^\alpha S_B\|_\infty=\prod_{j=1}^d \|\partial^{q_j}G_j\|_\infty$.

For $q_j\ge1$, Lemma~\ref{lem:barrier-fixed} gives $\|\partial^{q_j}G_j\|_\infty\lesssim \lambda^{q_j}$ and
Lemma~\ref{lem:prod-1d-fixed} gives $\|\partial^{q_j}G_j\|_1\lesssim \lambda^{q_j-1}$,
hence by interpolation $\|\partial^{q_j}G_j\|_p\lesssim \lambda^{\,q_j-\frac1p}$.
For $q_j=0$, $\|G_j\|_\infty\le 1$ and Lemma~\ref{lem:prod-1d-fixed} yields $\|G_j\|_1\le (u_j-\ell_j)+C/\lambda$, so
$\|G_j\|_p\le \|G_j\|_1^{1/p}\le \big((u_j-\ell_j)+C/\lambda\big)^{1/p}$.
Multiplying over $j$ gives the claim.
\end{proof}

\begin{lemma}\label{lem:union-separated}
Fix $\lambda\ge1$ and $\varepsilon=c_0/\lambda$ with $c_0>0$.
Let $B=\prod_{j=1}^d[\ell_j,u_j]$ and
\[
S_B(x)\ :=\ \prod_{j=1}^d G_j(x_j),\qquad
G_j(x_j):=\vartheta_{\lambda,\varepsilon}(u_j-x_j)\,\vartheta_{\lambda,\varepsilon}(x_j-\ell_j).
\]
Then
\[
\sum_{|\alpha|=d+1}\|\partial^\alpha S_B\|_{L^1(\mathbb R^d)}
\ \le\ C_d\sum_{r=1}^{d}\lambda^{\,d+1-r}\,\mathcal H^{d-r}\!\big(\Sigma_{d-r}(B)\big),
\]
where $C_d$ depends only on $d$ and $H$.
\end{lemma}

\begin{proof}
Write $\alpha=(\alpha_1,\dots,\alpha_d)$ with $|\alpha|=d{+}1$, and let
$J:=\{j:\alpha_j\ge1\}$ (the active axes), $r:=|J|\in\{1,\dots,d\}$.
By separability and Tonelli,
\[
\partial^\alpha S_B(x)\;=\;\prod_{j=1}^d \partial_{x_j}^{\,\alpha_j}G_j(x_j),
\qquad
\|\partial^\alpha S_B\|_{L^1}\;=\;\prod_{j=1}^d \big\|\partial_{x_j}^{\,\alpha_j}G_j\big\|_{L^1(\mathbb R)}.
\]

From Lemma~\ref{lem:barrier-fixed} and Lemma~\ref{lem:prod-1d-fixed}
(with $\varepsilon=c_0/\lambda$),
\[
\big\|\partial_{x_j}^{\,q}G_j\big\|_{L^1(\mathbb R)}
\ \le\
\begin{cases}
C\,\lambda^{\,q-1}, & q\ge1,\\[3pt]
(u_j-\ell_j)+C/\lambda, & q=0.
\end{cases}
\]
Thus, for the multiindex $\alpha$ with active set $J$,
\begin{equation}\label{eq:L1-product}
\|\partial^\alpha S_B\|_{L^1}
\ \le\
C^{\,r}\,\lambda^{\,\sum_{j\in J}(\alpha_j-1)}\,
\prod_{k\notin J}\Big((u_k-\ell_k)+\tfrac{C}{\lambda}\Big)
\ =\
C^{\,r}\,\lambda^{\,d+1-r}\,
\prod_{k\notin J}\Big((u_k-\ell_k)+\tfrac{C}{\lambda}\Big),
\end{equation}
since $\sum_{j\in J}\alpha_j=d+1$.

Let $a_k:=u_k-\ell_k$ and $\beta:=C/\lambda$. For fixed $J$,
\[
\prod_{k\notin J}(a_k+\beta)
\ =\
\sum_{L\subseteq J^c}\ \beta^{\,|L|}\ \prod_{k\notin J\cup L} a_k.
\]
Insert this into \eqref{eq:L1-product}:
\[
\|\partial^\alpha S_B\|_{L^1}
\ \le\
C^{\,r}\sum_{L\subseteq J^c}\ \lambda^{\,d+1-r}\,\beta^{\,|L|}\ \prod_{k\notin J\cup L} a_k
\ =\
C^{\,r}\sum_{L\subseteq J^c}\ \lambda^{\,d+1-(r+|L|)}\ \prod_{k\notin J\cup L} a_k.
\]
Define $r':=r+|L|\in\{r,\dots,d\}$. Grouping by $r'$,
\begin{equation}\label{eq:fixed-J-to-rprime}
\|\partial^\alpha S_B\|_{L^1}
\ \le\
\sum_{r'=r}^{d} C_d\,\lambda^{\,d+1-r'}\!\!\!\!
\sum_{\substack{L\subseteq J^c\\ |L|=r'-r}}
\ \prod_{k\notin J\cup L} a_k.
\end{equation}

For a fixed $J$ with $|J|=r$, the number of compositions of $d{+}1$ into $r$ strictly positive parts
\((\alpha_j)_{j\in J}\) is ${d \choose r-1}$; absorbing this (and the $C^{\,r}$) into $C_d$, the sum over all
$\alpha$ with $\mathrm{supp}(\alpha)=J$ yields
\[
\sum_{\substack{\alpha:\ |\alpha|=d+1\\ \mathrm{supp}(\alpha)=J}}
\|\partial^\alpha S_B\|_{L^1}
\ \le\
\sum_{r'=r}^{d} C_d\,\lambda^{\,d+1-r'}\!\!\!\!
\sum_{\substack{L\subseteq J^c\\ |L|=r'-r}}
\ \prod_{k\notin J\cup L} a_k.
\]

Now sum over all $J\subseteq\{1,\dots,d\}$ with $|J|=r$, and then over $r=1,\dots,d$.
For a fixed $r'$, the inner product $\prod_{k\notin J\cup L} a_k$ depends only on the union
$J':=J\cup L$ with $|J'|=r'$; each such $J'$ arises from finitely many pairs $(J,L)$,
which is absorbed into $C_d$. Hence
\[
\sum_{\substack{J\subset[d]\\ |J|=r}}
\ \sum_{\substack{L\subseteq J^c\\ |L|=r'-r}}
\ \prod_{k\notin J\cup L} a_k
\ \le\ C_d\sum_{\substack{J'\subset[d]\\ |J'|=r'}}
\ \prod_{k\notin J'} a_k.
\]
Recall that the \((d-r')\)-skeleton measure of an axis-aligned box satisfies
\[
\mathcal H^{d-r'}\!\big(\Sigma_{d-r'}(B)\big)
\ =\ \sum_{\substack{J'\subset[d]\\ |J'|=r'}} 2^{\,r'}\ \prod_{k\notin J'} a_k,
\]
because choosing $J'$ fixes which $r'$ coordinates are clamped to a face (each with two choices, $\ell$ or $u$), and the remaining coordinates span intervals of lengths $a_k$; overlaps of distinct faces have strictly lower dimension and therefore zero \(\mathcal H^{d-r'}\)-measure.
Therefore,
\[
\sum_{\substack{J'\subset[d]\\ |J'|=r'}} \prod_{k\notin J'} a_k
\ =\ 2^{-r'}\,\mathcal H^{d-r'}\!\big(\Sigma_{d-r'}(B)\big)
\ \le\ C_d\,\mathcal H^{d-r'}\!\big(\Sigma_{d-r'}(B)\big),
\]
absorbing $2^{-r'}$ into $C_d$.

Combining and summing $r'=1,\dots,d$,
\[
\sum_{|\alpha|=d+1}\|\partial^\alpha S_B\|_{L^1}
\ \le\
\sum_{r'=1}^{d} C_d\,\lambda^{\,d+1-r'}\,
\mathcal H^{d-r'}\!\big(\Sigma_{d-r'}(B)\big),
\]
as claimed.
\end{proof}

\begin{lemma}\label{lem:radon-master}
If $f\in C^{d+1}(\mathbb R^d)$ and $\partial^\alpha f\in L^1(\mathbb R^d)$ for all $|\alpha|=d{+}1$, then
\[
\|f\|_{\cR}
:=\int_{S^{d-1}}\int_{\mathbb R}\big|\partial_t^{\,d+1}(\mathcal R f)(\beta,t)\big|\,dt\,d\beta
\ \le\ C_d\sum_{|\alpha|=d+1}\|\partial^\alpha f\|_{L^1(\mathbb R^d)}.
\]
\end{lemma}
\begin{proof}
Fix $\beta\in S^{d-1}$ and write the Radon transform in coordinates
\[
(\mathcal R f)(\beta,t)=\int_{\beta^\perp} f(y+t\beta)\,d\mathcal H^{d-1}(y).
\]
Since $f\in C^{d+1}$, for each $y$ the map $t\mapsto f(y+t\beta)$ is $C^{d+1}$ with
$\partial_t^{\,d+1} f(y+t\beta) = (\beta\cdot\nabla)^{d+1}f(y+t\beta)$.
Moreover $(\beta\cdot\nabla)^{d+1}f$ is a finite linear combination of $\partial^\alpha f$, $|\alpha|=d+1$,
hence lies in $L^1(\mathbb R^d)$ by assumption. It follows (e.g. by standard differentiation-under-the-integral
criteria, or by mollification) that for a.e.\ $t$,
\[
\partial_t^{\,d+1}(\mathcal R f)(\beta,t)=\int_{\beta^\perp} (\beta\cdot\nabla)^{d+1} f(y+t\beta)\,d\mathcal H^{d-1}(y).
\]
Therefore, by triangle inequality and the change of variables $x=y+t\beta$,
\[
\int_{\mathbb R}\big|\partial_t^{\,d+1}(\mathcal R f)(\beta,t)\big|\,dt
\le \int_{\mathbb R}\int_{\beta^\perp} \big|(\beta\cdot\nabla)^{d+1} f(y+t\beta)\big|\,dydt
= \int_{\mathbb R^d}\big|(\beta\cdot\nabla)^{d+1} f(x)\big|\,dx.
\]
Next expand the directional derivative:
\[
(\beta\cdot\nabla)^{d+1} = \Big(\sum_{j=1}^d \beta_j \partial_j\Big)^{d+1}
= \sum_{|\alpha|=d+1} \binom{d+1}{\alpha}\,\beta^\alpha\,\partial^\alpha,
\]
so using $|\beta^\alpha|\le 1$,
\[
\int_{\mathbb R^d}\big|(\beta\cdot\nabla)^{d+1} f\big|
\le \sum_{|\alpha|=d+1}\binom{d+1}{\alpha}\,\|\partial^\alpha f\|_{L^1(\mathbb R^d)}
\le C_d \sum_{|\alpha|=d+1}\|\partial^\alpha f\|_{1},
\]
where $C_d:=\max_{|\alpha|=d+1}\binom{d+1}{\alpha}\le (d+1)!$ (or any comparable bound).
Finally integrate over $\beta\in S^{d-1}$ and absorb $|S^{d-1}|$ into $C_d$.
\end{proof}


\begin{theorem}[Bounded $\mathcal{R}\mathrm{TV}$ score for a single box]\label{thm:RTV-score-single}
With $S_B$ as in Lemma~\ref{lem:union-separated} and $\varepsilon=c_0/\lambda$,
\[
\|S_B\|_{\cR\mathrm{TV}}
\ \le\ C_{d}\sum_{r=1}^{d}\lambda^{\,d+1-r}\,\mathcal H^{d-r}\!\big(\Sigma_{d-r}(B)\big).
\]
\end{theorem}

\begin{proof}
Apply Lemma~\ref{lem:radon-master} to $f=S_B$ and Lemma~\ref{lem:union-separated}.
\end{proof}

\newpage

\newpage

%% file: ref.bib
@article{Parhi2019TheRO,
  title = {The Role of Neural Network Activation Functions},
  author = {Parhi, Rahul and Nowak, Robert D.},
  year = {2020},
  journal = {IEEE Signal Processing Letters},
  volume = {27},
  pages = {1779--1783},
  doi = {10.1109/LSP.2020.3027517},
}

@article{Parhi2020BanachSR,
  title = {Banach Space Representer Theorems for Neural Networks and Ridge Splines},
  author = {Parhi, Rahul and Nowak, Robert D.},
  year = {2021},
  journal = {Journal of Machine Learning Research},
  volume = {22},
  number = {43},
  pages = {1--40},
  url = {https://jmlr.org/papers/v22/20-583.html},
}

@misc{mao2024approximationratesshallowreluk,
      title={Approximation Rates for Shallow ReLU$^k$ Neural Networks on Sobolev Spaces via the Radon Transform}, 
      author={Tong Mao and Jonathan W. Siegel and Jinchao Xu},
      year={2024},
      eprint={2408.10996},
      archivePrefix={arXiv},
      primaryClass={stat.ML},
      url={https://arxiv.org/abs/2408.10996}, 
}

@book{kak_slaney,
  author       = {Kak, A C and Slaney, M},
  title        = {Principles of computerized tomographic imaging},
  url          = {https://www.osti.gov/biblio/5813672},
  place        = {United States},
  publisher    = {IEEE Service Center,Piscataway, NJ},
  year         = {1988},
  month        = {01}}

@InProceedings{savarese19a,
  title = 	 {How do infinite width bounded norm networks look in function space?},
  author =       {Savarese, Pedro and Evron, Itay and Soudry, Daniel and Srebro, Nathan},
  booktitle = 	 {Proceedings of the Thirty-Second Conference on Learning Theory},
  pages = 	 {2667--2690},
  year = 	 {2019},
  editor = 	 {Beygelzimer, Alina and Hsu, Daniel},
  volume = 	 {99},
  series = 	 {Proceedings of Machine Learning Research},
  month = 	 {25--28 Jun},
  publisher =    {PMLR},
  url = 	 {https://proceedings.mlr.press/v99/savarese19a.html}
}

@inproceedings{
Ongie2020A,
title={A Function Space View of Bounded Norm Infinite Width ReLU Nets: The Multivariate Case},
author={Greg Ongie and Rebecca Willett and Daniel Soudry and Nathan Srebro},
booktitle={International Conference on Learning Representations},
year={2020},
url={https://openreview.net/forum?id=H1lNPxHKDH}
}

@misc{rudin2019,
      title={Stop Explaining Black Box Machine Learning Models for High Stakes Decisions and Use Interpretable Models Instead}, 
      author={Cynthia Rudin},
      year={2019},
      eprint={1811.10154},
      archivePrefix={arXiv},
      primaryClass={stat.ML},
      url={https://arxiv.org/abs/1811.10154}, 
}

@misc{doshivelez2017,
      title={Towards A Rigorous Science of Interpretable Machine Learning}, 
      author={Finale Doshi-Velez and Been Kim},
      year={2017},
      eprint={1702.08608},
      archivePrefix={arXiv},
      primaryClass={stat.ML},
      url={https://arxiv.org/abs/1702.08608}, 
}

@misc{atrey2025demystifying,
      title={Demystifying the Accuracy-Interpretability Trade-Off: A Case Study of Inferring Ratings from Reviews}, 
      author={Pranjal Atrey and Michael P. Brundage and Min Wu and Sanghamitra Dutta},
      year={2025},
      eprint={2503.07914},
      archivePrefix={arXiv},
      primaryClass={cs.AI},
      url={https://arxiv.org/abs/2503.07914}, 
}

@misc{lovo2025tacklingaccuracyinterpretabilitytradeoffhierarchy,
      title={Tackling the Accuracy-Interpretability Trade-off in a Hierarchy of Machine Learning Models for the Prediction of Extreme Heatwaves}, 
      author={Alessandro Lovo and Amaury Lancelin and Corentin Herbert and Freddy Bouchet},
      year={2025},
      eprint={2410.00984},
      archivePrefix={arXiv},
      primaryClass={cs.LG},
      url={https://arxiv.org/abs/2410.00984}, 
}

@misc{neyshabur2015pathsgdpathnormalizedoptimizationdeep,
      title={Path-SGD: Path-Normalized Optimization in Deep Neural Networks}, 
      author={Behnam Neyshabur and Ruslan Salakhutdinov and Nathan Srebro},
      year={2015},
      eprint={1506.02617},
      archivePrefix={arXiv},
      primaryClass={cs.LG},
      url={https://arxiv.org/abs/1506.02617}, 
}

@misc{bartlett2017spectrallynormalizedmarginboundsneural,
      title={Spectrally-normalized margin bounds for neural networks}, 
      author={Peter Bartlett and Dylan J. Foster and Matus Telgarsky},
      year={2017},
      eprint={1706.08498},
      archivePrefix={arXiv},
      primaryClass={cs.LG},
      url={https://arxiv.org/abs/1706.08498}, 
}

@misc{jacot2020neuraltangentkernelconvergence,
      title={Neural Tangent Kernel: Convergence and Generalization in Neural Networks}, 
      author={Arthur Jacot and Franck Gabriel and Clément Hongler},
      year={2020},
      eprint={1806.07572},
      archivePrefix={arXiv},
      primaryClass={cs.LG},
      url={https://arxiv.org/abs/1806.07572}, 
}

@InProceedings{pmlr-v75-yarotsky18a,
  title = 	 {Optimal approximation of continuous functions by very deep ReLU networks},
  author =       {Yarotsky, Dmitry},
  booktitle = 	 {Proceedings of the 31st  Conference On Learning Theory},
  pages = 	 {639--649},
  year = 	 {2018},
  editor = 	 {Bubeck, Sébastien and Perchet, Vianney and Rigollet, Philippe},
  volume = 	 {75},
  series = 	 {Proceedings of Machine Learning Research},
  month = 	 {06--09 Jul},
  publisher =    {PMLR},
  url = 	 {https://proceedings.mlr.press/v75/yarotsky18a.html}
}

@InProceedings{Kontschieder_2015_ICCV,
author = {Kontschieder, Peter and Fiterau, Madalina and Criminisi, Antonio and Bulo, Samuel Rota},
title = {Deep Neural Decision Forests},
booktitle = {Proceedings of the IEEE International Conference on Computer Vision (ICCV)},
month = {December},
year = {2015}
}

@misc{frosst2017distillingneuralnetworksoft,
      title={Distilling a Neural Network Into a Soft Decision Tree}, 
      author={Nicholas Frosst and Geoffrey Hinton},
      year={2017},
      eprint={1711.09784},
      archivePrefix={arXiv},
      primaryClass={cs.LG},
      url={https://arxiv.org/abs/1711.09784}, 
}

@InProceedings{telgarsky2016benefits,
  title = 	 {benefits of depth in neural networks},
  author = 	 {Telgarsky, Matus},
  booktitle = 	 {29th Annual Conference on Learning Theory},
  pages = 	 {1517--1539},
  year = 	 {2016},
  editor = 	 {Feldman, Vitaly and Rakhlin, Alexander and Shamir, Ohad},
  volume = 	 {49},
  series = 	 {Proceedings of Machine Learning Research},
  address = 	 {Columbia University, New York, New York, USA},
  month = 	 {23--26 Jun},
  publisher =    {PMLR},
  url = 	 {https://proceedings.mlr.press/v49/telgarsky16.html}
}

@misc{eldan2016power,
      title={The Power of Depth for Feedforward Neural Networks}, 
      author={Ronen Eldan and Ohad Shamir},
      year={2016},
      eprint={1512.03965},
      archivePrefix={arXiv},
      primaryClass={cs.LG},
      url={https://arxiv.org/abs/1512.03965}, 
}

@misc{popov2019node,
      title={Neural Oblivious Decision Ensembles for Deep Learning on Tabular Data}, 
      author={Sergei Popov and Stanislav Morozov and Artem Babenko},
      year={2019},
      eprint={1909.06312},
      archivePrefix={arXiv},
      primaryClass={cs.LG},
      url={https://arxiv.org/abs/1909.06312}, 
}

@InProceedings{tanno2019ant,
  title = 	 {Adaptive Neural Trees},
  author =       {Tanno, Ryutaro and Arulkumaran, Kai and Alexander, Daniel and Criminisi, Antonio and Nori, Aditya},
  booktitle = 	 {Proceedings of the 36th International Conference on Machine Learning},
  pages = 	 {6166--6175},
  year = 	 {2019},
  editor = 	 {Chaudhuri, Kamalika and Salakhutdinov, Ruslan},
  volume = 	 {97},
  series = 	 {Proceedings of Machine Learning Research},
  month = 	 {09--15 Jun},
  publisher =    {PMLR},
  url = 	 {https://proceedings.mlr.press/v97/tanno19a.html}
}

@misc{hazimeh2020treeensemblelayer,
      title={The Tree Ensemble Layer: Differentiability meets Conditional Computation}, 
      author={Hussein Hazimeh and Natalia Ponomareva and Petros Mol and Zhenyu Tan and Rahul Mazumder},
      year={2020},
      eprint={2002.07772},
      archivePrefix={arXiv},
      primaryClass={cs.LG},
      url={https://arxiv.org/abs/2002.07772}, 
}

@inproceedings{
grande2023gradient,
title={{GRANDE}: Gradient-Based Decision Tree Ensembles for Tabular Data},
author={Sascha Marton and Stefan L{\"u}dtke and Christian Bartelt and Heiner Stuckenschmidt},
booktitle={The Twelfth International Conference on Learning Representations},
year={2024},
url={https://openreview.net/forum?id=XEFWBxi075}
}

@book{Mattila_1995, place={Cambridge}, series={Cambridge Studies in Advanced Mathematics}, title={Geometry of Sets and Measures in Euclidean Spaces: Fractals and Rectifiability}, publisher={Cambridge University Press}, author={Mattila, Pertti}, year={1995}, collection={Cambridge Studies in Advanced Mathematics}}

@article{Cybenko1989ApproximationBS,
  title={Approximation by superpositions of a sigmoidal function},
  author={George V. Cybenko},
  journal={Mathematics of Control, Signals and Systems},
  year={1989},
  volume={2},
  pages={303-314},
  url={https://api.semanticscholar.org/CorpusID:3958369}
}

@article{HORNIK1989359,
title = {Multilayer feedforward networks are universal approximators},
journal = {Neural Networks},
volume = {2},
number = {5},
pages = {359-366},
year = {1989},
issn = {0893-6080},
doi = {https://doi.org/10.1016/0893-6080(89)90020-8},
url = {https://www.sciencedirect.com/science/article/pii/0893608089900208},
author = {Kurt Hornik and Maxwell Stinchcombe and Halbert White}
}

@article{FUNAHASHI1989183,
title = {On the approximate realization of continuous mappings by neural networks},
journal = {Neural Networks},
volume = {2},
number = {3},
pages = {183-192},
year = {1989},
issn = {0893-6080},
doi = {https://doi.org/10.1016/0893-6080(89)90003-8},
url = {https://www.sciencedirect.com/science/article/pii/0893608089900038},
author = {Ken-Ichi Funahashi}
}

@ARTICLE{mhaskar,
  author={Mhaskar, H. N.},
  journal={Neural Computation}, 
  title={Neural Networks for Optimal Approximation of Smooth and Analytic Functions}, 
  year={1996},
  volume={8},
  number={1},
  pages={164-177},
  doi={10.1162/neco.1996.8.1.164}}

@article{Pinkus_1999, title={Approximation theory of the MLP model in neural networks}, volume={8}, DOI={10.1017/S0962492900002919}, journal={Acta Numerica}, author={Pinkus, Allan}, year={1999}, pages={143–195}}

@book{bateman1954tables,
  title={Tables of Integral Transforms: Based, in Part, on Notes Left by Harry Bateman, and Compiled by the Staff of the Bateman Manuscript Project. [A. Erd{\'e}lyi, Editor. W. Magnus, F. Oberhettinger, F. G. Tricomi, Research Associates]},
  author={Bateman Manuscript Project and Bateman, H. and Erd{\'e}lyi, A. and United States. Office of Naval Research},
  number={v. 1},
  lccn={54006214},
  series={California Institute of Technology: Bateman Manuscript Project},
  url={https://books.google.com/books?id=OLfZAAAAMAAJ},
  year={1954},
  publisher={McGraw-Hill}
}

@incollection{GELFAND19641,
title = {CHAPTER I - DEFINITION AND SIMPLEST PROPERTIES OF GENERALIZED FUNCTIONS},
editor = {I.M. Gel'fand and G.E. Shilov},
booktitle = {Properties and Operations},
publisher = {Academic Press},
pages = {1-151},
year = {1964},
isbn = {978-1-4832-2976-8},
doi = {https://doi.org/10.1016/B978-1-4832-2976-8.50007-6},
url = {https://www.sciencedirect.com/science/article/pii/B9781483229768500076},
author = {I.M. Gel'fand and G.E. Shilov}
}

@misc{kumar2025gapgaussianrkhsneural,
      title={A Gap Between the Gaussian RKHS and Neural Networks: An Infinite-Center Asymptotic Analysis}, 
      author={Akash Kumar and Rahul Parhi and Mikhail Belkin},
      year={2025},
      eprint={2502.16331},
      archivePrefix={arXiv},
      primaryClass={cs.LG},
      url={https://arxiv.org/abs/2502.16331}, 
}

@misc{kumar2024mirrordescentreproducingkernel,
      title={Mirror Descent on Reproducing Kernel Banach Spaces}, 
      author={Akash Kumar and Mikhail Belkin and Parthe Pandit},
      year={2024},
      eprint={2411.11242},
      archivePrefix={arXiv},
      primaryClass={cs.LG},
      url={https://arxiv.org/abs/2411.11242}, 
}
